\documentclass{article}

 \PassOptionsToPackage{numbers, compress}{natbib}
 \usepackage[main, preprint]{neurips_2026}
\usepackage[utf8]{inputenc} % allow utf-8 input
\usepackage[T1]{fontenc}    % use 8-bit T1 fonts
\usepackage{hyperref}       % hyperlinks
\usepackage[rgb,dvipsnames]{xcolor} % colors
\definecolor{linkblue}{RGB}{1, 103, 153}
\definecolor{linkteal}{RGB}{40, 105, 105}

\hypersetup{
    colorlinks=true,
    linkcolor=linkblue,
    citecolor=linkblue,
    urlcolor=linkteal
}
\usepackage{url}            % simple URL typesetting
\usepackage{booktabs}       % professional-quality tables
\usepackage{amsfonts}       % blackboard math symbols
\usepackage{nicefrac}       % compact symbols for 1/2, etc.
\usepackage{microtype}      % microtypography

\usepackage{wrapfig}
\usepackage{comment}
\usepackage{float}

\usepackage{amsmath}
\usepackage{amssymb}
\usepackage{mathtools}
\usepackage{amsthm}
\usepackage{enumitem}

\def\red{\color{red}}

\definecolor{mildcyan}{RGB}{7, 164, 165} % A soft, pale cyan (Pale Turquoise)

\usepackage[normalem]{ulem}
\usepackage{booktabs, makecell}
\usepackage{multirow}

\usepackage[capitalize,noabbrev]{cleveref}

\theoremstyle{plain}
\newtheorem{theorem}{Theorem}[section]
\newtheorem{proposition}[theorem]{Proposition}

\newtheorem{corollary}[theorem]{Corollary}
\theoremstyle{definition}

\theoremstyle{remark}
\newtheorem{remark}[theorem]{Remark}

\usepackage[textsize=tiny]{todonotes}
\graphicspath{ {./figs/} }
\usepackage{amsmath,bm,mathtools}

\usepackage[bb=pazo]{mathalpha} 

\def\bse{\begin{eqnarray*}}
\def\ese{\end{eqnarray*}}
\def\be{\begin{eqnarray}}
\def\ee{\end{eqnarray}}
\def\bsq{\begin{equation*}}
\def\esq{\end{equation*}}
\def\bq{\begin{equation}}
\def\eq{\end{equation}}

\DeclarePairedDelimiterX{\norm}[1]{\lVert}{\rVert}{#1}

\def\srcdata{{\mathcal{D}_\mathrm{src}}}
\def\tgtdata{{\mathcal{D}_\mathrm{tgt}}}
\def\calidata{{\mathcal{D}_\mathrm{cali}}}
\def\testdata{{\mathcal{D}_\mathrm{test}}}
\def\traindata{{\mathcal{D}_\mathrm{train}}}

\def\0{{\bf 0}}

\def\C{{\bf C}} %$\rvq \rmC \vq$

\def\q{{\bf q}}

\def\1{\mathbb{1}}
\def\bomega{{\boldsymbol\omega}}

\def\realNum{{\mathbb{R}}}

\def\spcX{{\mathcal{X}}}
\def\spcY{{\mathcal{Y}}}
\def\srcdata{{\mathcal{D}_\mathrm{src}}}
\def\tgtdata{{\mathcal{D}_\mathrm{tgt}}}
\def\ol{\overline}
\def\ul{\underline }

\def\LP{_\mathrm{LP}}
\def\GE{_\mathrm{GE}}

\def\FS{}
\def\g0{_{g^*}}
\def\wh{\widehat}
\def\wt{\widetilde}

\def\vol{\operatorname{volume}}

\def\pr{\mathbb{P}}

\def\MultiN{\operatorname{Multinomial}}

\def\wtsrg{\boldsymbol{\Omega}}

\def\crgn{\boldsymbol{\mathcal{C}}}
\def\qrgn{\boldsymbol{\mathcal{Q}}}
\DeclareMathOperator{\vect}{vec}

\def\eqref#1{(\ref{#1})}
\def\1{\bm{1}}

\def\trans{^{\top}}

\def\rvc{{\mathbf{c}}}

\def\rvv{{\mathbf{v}}}

\def\rmA{{\mathbf{A}}}
\def\rmB{{\mathbf{B}}}

\def\rmM{{\mathbf{M}}}
\def\rmN{{\mathbf{N}}}

\def\rmX{{\mathbf{X}}}

\def\rm0{{\mathbf{0}}}

\def\vb{{\bm{b}}}

\def\vd{{\bm{d}}}

\def\vv{{\bm{v}}}
\def\vw{{\bm{w}}}
\def\vx{{\bm{x}}}

\DeclareMathAlphabet{\mathsfit}{\encodingdefault}{\sfdefault}{m}{sl}
\SetMathAlphabet{\mathsfit}{bold}{\encodingdefault}{\sfdefault}{bx}{n}

\newcommand{\Var}{\mathrm{Var}}

\DeclareMathOperator*{\argmin}{arg\,min}

\def\bse{\begin{eqnarray*}}
\def\ese{\end{eqnarray*}}
\def\be{\begin{eqnarray}}
\def\ee{\end{eqnarray}}
\def\bsq{\begin{equation*}}
\def\esq{\end{equation*}}
\def\bq{\begin{equation}}
\def\eq{\end{equation}}

\def\boxit#1{\vbox{\hrule\hbox{\vrule\kern6pt\vbox{\kern6pt#1\kern6pt}\kern6pt\vrule}\hrule}}

\title{From Matrix Inversion to Constraints: Provably Tighter Confidence Regions for Importance Weights in Label Shift}
\author{%
Mushan Li\thanks{Equal contribution. Correspondence to Mushan Li <mzl5849@psu.edu>}~~, \quad
    Kihyun Han$^*$,\quad
    Yanyuan Ma\\
  Department of Statistics, The Pennsylvania State University,
  University Park, PA 16802, USA\\
  \texttt{\{mzl5849, kqh5716, yzm63\}@psu.edu} \\
}

\begin{document}
\maketitle
\begin{abstract}
Importance weights are essential in  domain adaptation under label shift, yet their utility is often undermined by the finite sample uncertainty associated with their estimation. Existing methods typically analyze this uncertainty through Gaussian elimination on interval-valued linear systems, which leads to overly conservative confidence regions and inefficient downstream applications. We  propose a paradigm shift from inversion-based inference to a direct matrix constraint framework. We use this framework to define a joint confidence region and extract marginal intervals via linear programming, deriving provably tighter bounds for importance weights while maintaining exact finite-sample validity. Furthermore, we analyze the confidence region's geometry and  provide the theoretical results for its diameter bounds. Evaluated across text, image, multimodal benchmarks,  including AGNews, MNIST, CIFAR-10, N24News,  and a real-world autonomous driving dataset, nuImages, our approach consistently yields shorter confidence intervals and smaller prediction sets than inversion-based methods.
\end{abstract}

%running title{From Matrix Inversion to Constraints for Label Shift Adaptation}

%Keywords:  Confidence Intervals, Confidence Regions, Label Shift, Linear Programming, Matrix Constraints, Prediction Sets, Uncertainty quantification

\section{Introduction} 
%All headings should be lower case (except for first word and proper nouns),
Label shift is a pervasive challenge in machine learning, occurring
when label distributions differ between source and target domains
while the conditional distribution of features remains invariant. This
scenario arises frequently in real-world applications.  For
  example, it occurs when
adapting models trained on historical medical data to new patient
populations, or when deploying classifiers across varying geographic
regions where class prevalences shift due to demographic or
environmental factors. In these cases, the ``importance weights'', the
ratios of target to source label probabilities, act as the essential
bridge for model adaptation, enabling tasks like weighted empirical
risk minimization and rigorous uncertainty quantification
\citep{lipton2018detecting, azizzadenesheli2019regularized,
  garg2020unified, podkopaev2021distribution, si2023pac}. 

The confusion matrix approach, exemplified by Black-Box Shift Estimation (BBSE) \citep{lipton2018detecting} and Regularized Learning under Label Shifts (RLLS) \citep{azizzadenesheli2019regularized}, is the most widely adopted technique because its statistical consistency is independent of classifier calibration. While these methods are effective for point estimation, they generally overlook the inherent statistical uncertainty in the importance weights estimation.
When uncertainty is taken into account, it is mostly considered in the
asymptotic sense \citep{tian2023elsa}. Notably, \citet{saerens2002adjusting} 
and \citet{alexandari2020maximum} devised maximum likelihood estimators to construct confidence intervals, which rely on the infinite data assumption and, crucially, necessitate well-calibrated classifiers \citep{garg2020unified}.
In finite-sample settings, however, the asymptotic approximations can be
unreliable and may fail to account for uncertainties that propagate to downstream tasks. 

 We are only aware of one work \citep{si2023pac} that tries to establish
  uncertainty quantification in the finite sample setting, 
  where they quantify the importance  weights uncertainty by bounding errors in the
  Gaussian Elimination (GE) procedure.
\citet{si2023pac} considered an interval-version of  BBSE,
where they replaced every element in the confusion matrix
  and every element in the target predicted label proportion vector
by its corresponding confidence interval, 
and modified the traditional GE to solve this interval-valued linear
system. Thus, this method still builds on the  ``matrix inversion''
mindset. By performing row reduction or inversion on intervals, they
introduce unnecessary relaxations and ``noise amplification'', which
leads to loose and suboptimal confidence regions. In practice, this
means that it becomes overly
conservative. In our numerical analysis, we find that the GE procedure
  frequently overestimates the uncertainty
and offers little practical value for decision-making. 

In this work, we propose  Matrix Constraints Linear Programming
  (MaC-LP), a new uncertainty
  quantification method for the importance weights.
It represents a fundamental
  paradigm shift: from matrix inversion to matrix constraints.
  Rather than trying to solve the system algebraically
  or relying on asymptotic approximations, we express the feasible range
  of importance weights as a system of matrix constraints (MaC) derived from
  the confidence intervals of the observations. By applying linear programming (LP) to optimize over these constraints, we compute the provably shorter confidence intervals for each one of importance weights.

MaC-LP ensures exact finite-sample validity. The
  resulting tighter 
boundaries have a direct, transformative impact on the efficiency of
downstream tasks. By sharpening the confidence intervals of the 
  estimated
importance weights, we allow for more precise uncertainty-aware
predictions. This ensures that even in the label shifted
environments, the estimation and inference in the target domain remain valid in
  finite samples, while
  avoiding excessive conservatism.
Our specific contributions are as follows:
\begin{itemize}[leftmargin=.65cm]
    \item
    Constraint-based framework: We introduce a novel 
      matrix-inversion-free approach  to formulate confidence
      regions for importance weights. We propose to construct the confidence
      regions via  the matrix constraints (MaC) and obtain the marginal confidence intervals by the
    linear programming (LP).
    \item
 Optimality and superiority over asymptotic and inversion method:
  We prove that
      our confidence region has exact finite-sample
      validity, in contrast to the asymptotic based methods.
 We prove that our result achieves the
  smallest confidence region possible in finite-sample regimes given the elementwise
confidence intervals of the confusion matrix and the predicted label proportions. We prove that our linear programming-based confidence intervals are no wider than the ones derived from Gaussian elimination.
    \item
      Enhanced downstream efficiency:
      We  prove that the tighter confidence region
      directly translates to precise and smaller
      uncertainty-aware prediction sets in the target domain under
      label shift. 
    \item
    Empirical validation: Through experiments on diverse benchmarks,
    including text (AGNews), images (MNIST, CIFAR-10), multimodal
    data (N24News), and a real-world autonomous driving dataset (nuImages),  we show significant reductions in both confidence
    interval lengths of the importance weights and the size of
    downstream prediction sets compared to those by Gaussian Elimination.
\end{itemize}

\section{Problem formulation and our goal}
\label{sec:preliminaries}
In this section, we introduce the basic notation, motivate the role of importance weights in correcting distribution mismatch under label shift, and present our primary goal of this work.

\paragraph{Basic notation.} Let $\spcX$ and $\spcY = \{1, \dots,
K\}\equiv[K]$ denote the feature and label spaces,
respectively. In the unsupervised label shift setting,
we are given a labeled source dataset $\srcdata = \{(\vx_s,
y_s)\}_{s=1}^{m}$ drawn i.i.d. from the source distribution $P$, and an
unlabeled target dataset $\tgtdata = \{\vx_{m +t}\}_{t=1}^{n}$ drawn from the
target distribution $Q_X$, where $Q_X$ is an $\rmX$-marginal of the target distribution $Q$, and
$(m, n)$ are finite sample sizes. Both $P$ and $Q$ are defined on
$\spcX \times\spcY$ with corresponding probability density/mass
functions $p(\cdot)$ and $q(\cdot)$.
Specifically, $p(\vx, y)$, $p(\vx|y)$, $p(\vx)$ and $p(y)$ denote the
joint distribution, the class-conditional distribution, the marginal
feature and  the marginal label  distribution, respectively, under
$P$, with analogous notations under $Q$. 
In addition, let $\pr$ denote probability, with the underlying
measure/distribution indicated via subscript when necessary (e.g.,
$\pr_{(\rmX,Y)\sim P}, \pr_{\rmX\sim Q_X}$). 

\paragraph{Importance weights under label shift.} The \textit{label
  shift} assumption posits that class-conditional distributions are
invariant across domains, i.e., $q(\vx|y) = p(\vx|y)$ for all $y \in
\spcY$, while the marginal label distributions differ, i.e., $p(y) \neq q(y)$. 
A direct consequence of this assumption is that the joint
distributions are related via a ratio of their label proportions: 
\bse
q(\vx, y) = p(\vx \mid y) q(y)
    = p(\vx\mid y) p(y) \frac{q(y)}{p(y)} = p(\vx, y) \omega_y,
\ese
where $\omega_y \equiv q(y)/p(y)$ represents the \textit{importance
  weights}. These weights, denoted by $\bomega =
(\omega_1,...,\omega_K)\trans$, serve as the fundamental bridge
between domains. For example, it allows us to express target
expectations in terms of the source distribution:
$\mathbb{E}_Q\{h(\rmX, Y)\} = \mathbb{E}_P\{\omega_Y h(\rmX, Y)\}$,
where $ \mathbb{E}$ denotes the expectation operator and $h$ is a
scalar function defined on $\spcX \times\spcY$.  
In practice, while the true importance weights $\bomega^*$ are unknown, using $\wh\bomega$ estimated
by finite samples $\srcdata$ and $\tgtdata$ is an essential way for
downstream tasks such as \textit{weighted conformal prediction}
\citep{podkopaev2021distribution}, where $\wh\bomega$ is used to
recalibrate non-conformity scores to maintain valid coverage
guarantees under label shift. 

\setlength{\intextsep}{-10pt}
\begin{wrapfigure}{r}{0.5\textwidth}
    \begin{center}
    \includegraphics[width=1.02\linewidth]{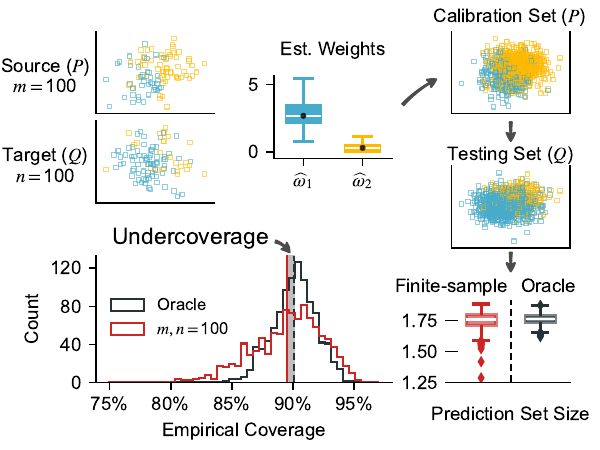}
    \caption{ Undercoverage of weighted conformal prediction using $\wh\bomega$.
    Step-histograms show frequency of empirical coverage over 1,000 repetitions. While the Oracle (black) centers on the 90\% nominal confidence level, point estimation of importance weights (red) introduces a systematic \textit{coverage gap}, even though the estimation is unbiased (true importance weights are represented by black dots in the boxplots of Est.\,Weights). 
    At $m,n=100$, estimation error leads to significant undercoverage. 
    This motivates the need for quantifying the uncertainty in
    estimated importance weights to ensure its application in the
    finite-sample setting.}
        \label{fig:1}
    \end{center}
\end{wrapfigure}

\paragraph{Our goal: improve inference for importance weights.}
While obtaining a point estimate $\wh\bomega$ is the standard approach to addressing label shift, its utility is compromised by its uncertainty. In the finite-sample case, when the source and/or target sample sizes ($m, n$) are small, the estimate $\wh\bomega$ becomes highly sensitive to the data noise.

Using prediction set as a representative example, Figure \ref{fig:1} illustrates that replacing $\bomega^*$ by $\wh\bomega$ for downstream tasks, such as weighted conformal prediction, can lead to a coverage gap. It is because when the uncertainty in $\wh\bomega$ is ignored, the resulting coverage probability can fall strictly below the nominal level 90\% when $\srcdata$ and $\tgtdata$ are small, causing prediction sets to be unreliably narrow and miss the ground truth more often than permitted.

This example motivates us to move beyond a point estimate and instead consider a region of importance weights, $\wtsrg\FS$, such that $\pr(\bomega^* \in \wtsrg) \ge 1-\delta$. 
Meanwhile, to avoid an excessively conservative
  prediction set, the confidence region $\wtsrg$ should be as small as
  possible. 
  Also, a tighter confidence region improves sample efficiency, requiring fewer observations to achieve the same target precision at a fixed confidence level.
In the following sections, we introduce the MaC-LP framework to derive these confidence regions through a novel matrix constraint approach.

\section{Finite-sample uncertainty quantification under label shift}
To facilitate the description of our approach, we briefly review the confusion matrix framework for label shift estimation \citep{lipton2018detecting, garg2020unified}. While this framework was previously proposed for point estimation of the importance weights, we extend it to define a direct matrix constraint system that captures finite-sample uncertainty.

\subsection{Preparation: confusion matrix framework}\label{sec:confusion-matrix-approach}
Suppose that we have a classifier $g$, defined as $g:\spcX \to [K]$.
Define the $K\times K$ (joint) confusion matrix
$\C^*=(C_{ij}^*)_{i,j\in[K]}$ by
$C_{ij}^* \equiv \pr_{(\rmX,Y)\sim P}\{g(\rmX)=i,\ Y=j\}$
for $i,j\in[K]$.
In the target domain, define the predicted label proportions
$\q^*=(q_1^*,\ldots,q_K^*)^\top$ by 
$q_k^* \equiv \pr_{\rmX\sim Q_X}\{g(\rmX)=k\}$ for $k\in[K]$.
By the  label shift assumption, we have
$\pr_{(\rmX,Y)\sim P}\{g(\rmX)=i \mid Y=j\}=\pr_{(\rmX,Y)\sim Q}\{g(\rmX)=i \mid Y=j\}$ for all $i,j\in[K]$.
Therefore, for each $i\in[K]$, we obtain that $\sum_{j=1}^K C_{ij}^*\omega_j^*=q_i^*$,
which implies $\C^*\bomega^*=\q^*$. We provide the detailed proof in Appendix~\ref{sec:Comegaq}.

In the source domain, given finite samples $\srcdata = \{(\vx_s, y_s)\}_{s=1}^{m}$, for $i,j\in[K]$, let
$M_{ij}=\sum_{s=1}^m I\{g(\vx_s)=i,\ y_s=j\}$,
where $I\{\cdot\}$ is the indicator function. 
Denote $\rmM = (M_{ij})_{i,j\in[K]}$ which is  a $K\times K$ matrix. Note that $\sum_{i=1}^K\sum_{j=1}^K M_{ij}=m$.
In the target domain, given finite samples $\tgtdata = \{\vx_{m+t}\}_{t=1}^{n}$,
 for $k\in[K]$, let
$N_k=\sum_{t=1}^n I\{g(\vx_{m+t})=k\}$,
and denote $\rmN=(N_1,\ldots,N_K)^\top$.
Similarly, we have $\sum_{k=1}^K N_k=n$.
\citet{lipton2018detecting} derived the plug-in estimator $\wh\bomega=\wh\C^{-1}\wh\q$ by taking $\wh\C= m^{-1}\rmM$ and $\wh\q=n^{-1}\rmN$ (see Appendix~\ref{sec:review-est} for an overview of importance weights estimation).
In contrast, to construct the confidence region for $\bomega^*$, we
marginalize each element of %multinomial
 random matrix
  $\rmM$ and random vector $\rmN$ to get $M_{ij}\sim \mathrm{Binomial}(m,C_{ij}^*)$ and $N_k\sim \mathrm{Binomial}(n,q_k^*)$ for each $i,j,k\in[K]$.
Using standard confidence intervals for binomial probability, such as Clopper-Pearson interval or Hoeffding's inequality-based interval (details in Appendix~\ref{sec:ci-binomial}), we can obtain the elementwise confidence intervals $[\ul C_{ij},\ol C_{ij}]$ for $C_{ij}^*$ and
$[\ul q_k,\ol q_k]$ for $q_k^*$ with coverage levels
$1-\delta_{ij}$ and $1-\delta_k$, respectively, where $\delta_{ij}, \delta_k\ge 0$.

\subsection{MaC-LP: inference via matrix constraints and linear programming}
\label{sec:mac_lp}
Following the notation from Section \ref{sec:confusion-matrix-approach}, given the confidence intervals $[\ul C_{ij}, \ol C_{ij}]$ and $[\ul q_k, \ol q_k]$ for all $i,j,k \in [K]$, we denote the collections of these endpoints as $\ul\C = (\ul C_{ij})_{i,j\in[K]}$, $\ol\C = (\ol C_{ij})_{i,j\in[K]}$, $\ul\q = (\ul q_k)_{k\in[K]}$, and $\ol \q = (\ol q_k)_{k\in[K]}$.
In addition, for any vectors or matrices $\mathbf{a} =
(a_1,\ldots,a_L)$ and $\mathbf{b} = (b_1,\ldots,b_L)$ of the same
size, denote $\mathbf{a} \le \mathbf{b}$ if $a_\ell \le b_\ell$ for
all $\ell\in[L]$, and let the same definitions apply to $\mathbf{a} <
\mathbf{b}$, $\mathbf{a} \ge \mathbf{b}$, and $\mathbf{a} >
\mathbf{b}$. 
With these notations, we recast the confidence intervals as $\C^* \in
[\ul\C, \ol\C] \equiv \crgn$ and $\q^* \in [\ul\q, \ol\q] \equiv
\qrgn$. 

With these confidence intervals in hand, we aim to construct the $(1-\delta)$-level confidence region for the importance weights
 $\bomega^*$.
To this end, we collect all $\bomega$ vectors that can be generated by an admissible pair $(\C,\q)\in\crgn\times\qrgn$ and define
\bse
\wtsrg_0
\equiv
\big\{\bomega:\exists\,\C\in\crgn,\,\exists\,\q\in\qrgn\text{ such that } \C\bomega=\q,\,\bomega\ge\0\big\}.
\ese
By the definition of $\wtsrg_0$, $\bomega^*\in \wtsrg\FS$ provided that $\C^*\in\crgn$ and $\q^*\in\qrgn$.

Let $\delta=\sum_{i, j \in[K]} \delta_{ij} + \sum_{k \in[K]} \delta_{k}$.
If the intervals $[\ul C_{ij}, \ol C_{ij}]$ and $[\ul q_{k}, \ol q_{k}]$
are valid confidence intervals for $C_{ij}^*$ and $q_k^*$, respectively,
then 
the union bound imply that
$\wtsrg_0$ is a valid confidence region for $\bomega^*$, as summarized in Proposition~\ref{prop:unionlaw}. The proof is in Appendix~\ref{sec:unionlaw}.

\begin{proposition}\label{prop:unionlaw}
Let $\delta=\sum_{i, j \in[K]} \delta_{ij} + \sum_{k \in[K]} \delta_{k}$.
Suppose $\pr(C_{ij}^* \in [\ul C_{ij}, \ol C_{ij}]) \ge 1-\delta_{ij}$ for $i,j\in [K]$ and $\pr(q_k^* \in [\ul q_{k}, \ol q_{k}]) \ge 1-\delta_k$ for $k \in [K]$. Then $\wtsrg_0$ is a valid $(1-\delta)$-level confidence
region for $\bomega^*$.
\end{proposition}

However, computing $\bomega$ by solving the linear system $\C\bomega=\q$ for 
infinitely many pairs $(\C,\q)\in\crgn\times\qrgn$ is numerically inconvenient. Instead, we directly impose the linear constraints relying on the linear system $\C\bomega=\q$.
Given a $\C \in \crgn$, if $\C\bomega=\q \in \qrgn$, then  the quantity $(\C\bomega)_i=\sum_{j=1}^K C_{ij}\omega_j$ must lie in $[\ul q_i,\ol q_i]$ for each $i\in[K]$.
Since $\C\in[\ul\C,\ol\C]$, a necessary condition for this to be
possible is that 
$[\sum_{j=1}^K \ul C_{ij}\omega_j,\ \sum_{j=1}^K \ol C_{ij}\omega_j]$ overlaps with $[\ul q_i,\ol q_i]$, i.e., for $i\in[K]$,
\bse
\sum_{j=1}^K \ol C_{ij}\omega_j \ge \ul q_i
\text{ and }
\sum_{j=1}^K \ul C_{ij}\omega_j \le \ol q_i.
\ese
The following theorem shows that the above condition is not only necessary but also
sufficient, and hence $\wtsrg_0$ admits an equivalent
characterization via explicit linear inequalities. 
Throughout this work, we use MaC and $\wtsrg$ interchangeably to refer to the construction in Theorem~\ref{thm:lp}. 
The proof of Theorem~\ref{thm:lp} is provided in Appendix~\ref{sec:lp}.
\begin{theorem}\label{thm:lp}
Let $\crgn=[\ul\C,\ol\C]$ and $\qrgn=[\ul\q,\ol\q]$.
Then
\be
\wtsrg\FS
=
\big\{\bomega:\ol\C\bomega\ge\ul\q,\,\ul\C\bomega\le\ol\q,\,\bomega\ge\0\big\}.\label{eq:thmlp}
\ee
 In particular, $\wtsrg\FS$ is a convex set.
\end{theorem}

\begin{remark}[Connection to the Oettli-Prager Theorem]
The explicit solution derived in Theorem~\ref{thm:lp} happens to coincide with the classical Oettli-Prager Theorem \citep{oettli1964compatibility} restricted to the nonnegative orthant. We detail the connection in Appendix \ref{sec:connect_OPT}\footnote{We thank an anonymous reviewer whose comments led us to unearth the connection to the Oettli-Prager Theorem.}.
\end{remark}

Although MaC offers a closed-form description of this region, its shape may be irregular.
In practice, when a regular shape such as an axis-aligned
hyperrectangle is preferred, we can
apply the linear programming on the linear-inequality characterization, $\wtsrg$, in
Theorem~\ref{thm:lp}. It computes the classwise bounds of
$\wtsrg\FS$ by solving $2K$ linear programs: for each $k\in[K]$, 
\be
\ul\omega_k=\min\omega_k
\hbox{ subject to } \ol\C\bomega\ge\ul\q,\ \ul\C\bomega\le\ol\q,\ \bomega\ge\0,\label{eq:lp1}\\
\ol\omega_k=\max\omega_k
\hbox{ subject to } \ol\C\bomega\ge\ul\q,\ \ul\C\bomega\le\ol\q,\ \bomega\ge\0.\label{eq:lp2}
\ee
Collecting these bounds yields the smallest axis-aligned hyperrectangle (``squeezing box'') that contains $\wtsrg\FS$ as
\be\label{eq:omega-lp}
\wtsrg\LP\equiv\prod_{k=1}^K[\ul\omega_k,\ol\omega_k].
\ee
Since the objective $\omega_k$ in \eqref{eq:lp1}--\eqref{eq:lp2} is
linear in $\bomega$ and the constraints are also linear, we can compute
$\wtsrg\LP$ using off-the-shelf linear programming methods (e.g., the
simplex algorithm).  
Figure~\ref{fig:workflow} in Appendix~\ref{sec:workflow} summarizes the construction of $\wtsrg\FS$
and $\wtsrg\LP$ in a binary-classification toy example.

Since $\wtsrg\LP$ is defined as an axis-aligned hyperrectangle, its side lengths serve as a natural metric for uncertainty, enabling a direct comparison of confidence region sizes across different methodologies. A closely related approach by \citet{si2023pac} constructs a hyperrectangular region $\wtsrg\GE$ for $\bomega^*$ by evaluating the inversion $\C^{-1}\q$ for all $\C \in \crgn$ and $\q\in \qrgn$ using interval-based Gaussian elimination (see Appendix~\ref{sec:gaussian}). In Corollary~\ref {cor:www}, we establish the formal set inclusions between the fundamental region $\wtsrg\FS$, our proposed $\wtsrg\LP$, and the conservative $\wtsrg\GE$, while maintaining $1-\delta$ coverage. 
The proof of Corollary~\ref{cor:www} is provided in
Appendix~\ref{sec:www}. A two-class example illustrating why interval Gaussian elimination is conservative is given in Appendix~\ref{sec:ge_example}.
\begin{corollary}\label{cor:www}
When $\wtsrg\GE$ exists, $\wtsrg\FS\subset\wtsrg\LP\subset\wtsrg\GE$.
Under the conditions of Proposition~\ref{prop:unionlaw},
\bse
1-\delta\le\pr(\bomega^*\in\wtsrg\FS)\le\pr(\bomega^*\in\wtsrg\LP)
\le\pr(\bomega^*\in\wtsrg\GE).
\ese
\end{corollary}

Note that every optimality statement about $\wtsrg$ is conditional on the elementwise confidence intervals $\crgn\times\qrgn$. Additional simplex or joint multinomial constraints on $(\C,\q)$ can shrink the region further, but at the cost of convexity and of the linear characterization (Appendix~\ref{app:extended_discussion}).

\subsection{Analysis: diameter of the confidence regions}
As $\wtsrg$ represents the tightest confidence region for $\bomega^*$ given the uncertainty sets $\crgn$ and $\qrgn$, it is essential to characterize its geometric scale.
We summarize the size of $\wtsrg\FS\subset\mathbb{R}^K$ by its diameter under
$\|\cdot\|_\infty$ and $\|\cdot\|_2$ norms.
For a generic set $\mathcal{S}\subset\realNum^K$, define
${\mathrm{diam}}_\infty(\mathcal{S})\equiv\sup_{\bomega,\bomega'\in\mathcal{S}}\|\bomega-\bomega'\|_\infty$
and
${\mathrm{diam}}_2(\mathcal{S})\equiv\sup_{\bomega,\bomega'\in\mathcal{S}}\|\bomega-\bomega'\|_2$. 
Since $\wtsrg\LP$ is the smallest hyperrectangle containing
$\wtsrg\FS$, we have
${\mathrm{diam}}_\infty(\wtsrg\FS)={\mathrm{diam}}_\infty(\wtsrg\LP)=\max_{k\in[K]}(\ol\omega_k-\ul\omega_k)$.

The next theorem provides upper bounds on
${\mathrm{diam}}_\infty(\wtsrg\FS)$ and ${\mathrm{diam}}_2(\wtsrg\FS)$
in terms of $\ul\C$, $\ol\C$, $\ul\q$ and $\ol\q$.
For a generic matrix $\rmA=(A_{ij})_{i,j\in[K]}$, let
$\|\rmA\|_\infty\equiv\max_{i\in[K]}\sum_{j=1}^K|A_{ij}|$ and
$\|\rmA\|_2\equiv\sup_{\|\vv\|_2=1}\|\rmA\vv\|_2$.
The proof of Theorem~\ref{thm:diameter} is provided in Appendix~\ref{sec:diameter}.

\begin{theorem}\label{thm:diameter}
Let $\ul\C\in\realNum^{K\times K}$ be invertible.
If $\|\ul\C^{-1}\|_\infty\|\ol\C-\ul\C\|_\infty<1$, then 
\bse
{\mathrm{diam}}_\infty(\wtsrg\FS)\le\frac{\|\ul\C^{-1}\|_\infty^2\|\ol\C-\ul\C\|_\infty\|\ol\q\|_\infty}{\big(1-\|\ul\C^{-1}\|_\infty\|\ol\C-\ul\C\|_\infty\big)^2}
+\frac{\|\ul\C^{-1}\|_\infty\|\ol\q-\ul\q\|_\infty}{1-\|\ul\C^{-1}\|_\infty\|\ol\C-\ul\C\|_\infty}.
\ese
If $\|\ul\C^{-1}\|_2\|\ol\C-\ul\C\|_2<1$, then 
\bse
{\mathrm{diam}}_2(\wtsrg\FS)\le\frac{\|\ul\C^{-1}\|_2^2\|\ol\C-\ul\C\|_2\|\ol\q\|_2}{\big(1-\|\ul\C^{-1}\|_2\|\ol\C-\ul\C\|_2\big)^2}
+\frac{\|\ul\C^{-1}\|_2\|\ol\q-\ul\q\|_2}{1-\|\ul\C^{-1}\|_2\|\ol\C-\ul\C\|_2}.
\ese
\end{theorem}

\begin{remark}
The assumptions $\|\ul\C^{-1}\|_\infty\|\ol\C-\ul\C\|_\infty<1$ and
$\|\ul\C^{-1}\|_2\|\ol\C-\ul\C\|_2<1$ imply that $\C^{-1}$
exists  for all $\C\in\crgn$, and  $\C\mapsto\C^{-1}$ is a smooth function
on $\crgn$.
These conditions are not stringent when the confidence intervals for
$\C^*$ are sufficiently tight, since $\|\ol\C-\ul\C\|$ shrinks with
sample size $m$ while $\|\ul\C^{-1}\|$ is bounded whenever $\ul\C$ is well-conditioned (e.g., nearly diagonal or diagonally dominant). We illustrate the usage of Theorem~\ref{thm:diameter} by
deriving an explicit bound in terms of $(m,n,K)$ in Appendix~\ref{app:diameter}. 
\end{remark}

\subsection{Downstream task: finite sample prediction set}\label{sec:finite-sample-prediction}
Benefiting from the confidence region of $\bomega^*$, we propose two
ways of constructing a prediction set, 
conformal prediction \citep{vovk2005algorithmic} and probably
approximately correct (PAC) prediction 
\citep{valiant1984theory}. We provide finite-sample guarantees for
these prediction sets. In existing works, 
finite-sample coverage is typically shown only when the true
importance weights $\bomega^*$ are known, and
when $\bomega^*$ is replaced by an estimator $\wh\bomega$, results are
generally asymptotic. We 
improve these methods by obtaining prediction sets with a prescribed
coverage level in finite samples  instead of asymptotically,
while allowing unknown $\bomega^*$ by incorporating 
a confidence region for $\bomega^*$. We show the conformal prediction results here, while presenting the PAC prediction results in Appendix \ref{app:PAC}.

\subsubsection{Conformal prediction}\label{sec:conformal-prediction}
We develop an uncertainty-aware conformal prediction set under label shift. We first
review the weighted split conformal construction for a fixed candidate
weight vector $\bomega$, and then incorporate a confidence region for
$\bomega^*$ to obtain finite-sample coverage when $\bomega^*$ is
unknown. 

We consider the split conformal prediction setting, where the source data is divided into a calibration set
$\calidata=\{z_i=(\wt\vx_i,\wt y_i):i=1,\ldots,m_1\}$ and a training set $\traindata$, where the nonconformity
score $r(\wt\vx,\wt y)\in[0,1]$ is trained on. The prediction set
is constructed using the calibration set $\calidata$. Let $z_0=(\wt\vx_0,\wt y_0)$
denote a new target-domain pair, where $\wt\vx_0$ is observed and
$\wt y_0\in[K]$ is the unknown label. Under the label shift assumption,
the calibration data set $\calidata$ and $z_0$ satisfy the weighted
exchangeability condition of \citet{tibshirani2019conformal}, 
\bse
q(z_0,z_1,\ldots,z_{m_1})=p(z_0,z_1,\ldots,z_{m_1})\prod_{i=0}^{m_1}\omega_{\wt y_i}^*,
\ese
where $p(z_0,\ldots,z_{m_1})=p(z_{\sigma(0)},\ldots,z_{\sigma(m_1)})$
for any permutation
$\sigma:\{0,\ldots,m_1\}\to\{0,\ldots,m_1\}$. This weighted
exchangeability motivates calibrating $r(\wt\vx,\wt y)$ using the weighted
empirical distribution of calibration scores. Let $r_i=r(\wt\vx_i,\wt y_i)$
for $i=1,\ldots,m_1$. For any candidate $\bomega$, we define the level
$(1-\alpha)$ conformal prediction set
\bse
F_{\mathrm{CP}}(\wt\vx_0;\bomega)=\{\wt y_0\in[K]:r(\wt\vx_0,\wt y_0)\le\tau_{\mathrm{CP}}(\wt y_0;\bomega)\},
\ese
where $\tau_{\mathrm{CP}}(\wt y_0;\bomega)$ is given by
\bse
\tau_{\mathrm{CP}}(\wt y_0;\bomega)=Q_{1-\alpha}\left(\sum_{i=1}^{m_1}\kappa_{r_i}\frac{\omega_{\wt y_i}}{\sum_{j=1}^{m_1}\omega_{\wt y_j}+\omega_{\wt y_0}}\right.
\left.+\kappa_1\frac{\omega_{\wt y_0}}{\sum_{j=1}^{m_1}\omega_{\wt y_j}+\omega_{\wt y_0}}\right),
\ese
with $\kappa_r$ denoting the Dirac measure at $r$, and
$Q_{1-\alpha}(\mu)$ denoting the $(1-\alpha)$-th quantile of a probability measure $\mu$.

When the true importance weight $\bomega^*$ is known, the prediction
set $F_{\mathrm{CP}}(\wt\vx_0;\bomega^*)$ achieves $1-\alpha$ coverage by
Theorem~2 of \citet{podkopaev2021distribution}. However, $\bomega^*$
is unknown in practice. Let $\wtsrg_0$ be a confidence region for
$\bomega^*$. To obtain a prediction set that remains valid for all
$\bomega\in\wtsrg_0$, we enlarge the threshold by taking a supremum:
\be\label{eq:tcp-omega0}
\tau_{\mathrm{CP}}(\wt y_0;\wtsrg_0)=\sup_{\bomega\in\wtsrg_0}\tau_{\mathrm{CP}}(\wt y_0;\bomega),
\ee
and define the robust conformal prediction set
\be\label{eq:fcp-omega0}
F_{\mathrm{CP}}(\wt\vx_0;\wtsrg_0)=\{\wt y_0\in[K]:r(\wt\vx_0,\wt y_0)\le\tau_{\mathrm{CP}}(\wt y_0;\wtsrg_0)\}. 
\ee
Compared to the known $\bomega^*$ case, the construction in
\eqref{eq:fcp-omega0} enlarges the prediction set due to not knowing $\bomega^*$.
Nevertheless, we can still control the prediction level at
$1-\delta-\alpha$, as established in
Theorem~\ref{thm:cp-omega0}. Intuitively, coverage can fail only if
$\bomega^*\notin\wtsrg_0$ (with probability at most $\delta$) or if
the conformal event fails (with probability $\alpha$). See
Appendix~\ref{sec:cp-omega0} for the proof. 

\begin{theorem}\label{thm:cp-omega0}
Suppose that an independent training set
$\traindata$ is used to train the classifier. Under label shift assumption, suppose $\calidata \sim P^{m_1}$ i.i.d., $(\wt\rmX_0,\wt Y_0)\sim Q$ independent of
$\calidata$, and
$\wtsrg_0=\wtsrg_0(\srcdata,\tgtdata)$ satisfies
$\pr(\bomega^*\in\wtsrg_0)\ge1-\delta$ given the training set. Then
\bse
\pr\{\wt Y_0\in F_{\mathrm{CP}}(\wt\rmX_0;\wtsrg_0)\}\ge 1-\delta-\alpha,
\ese
where $\pr$ is taken jointly over $\calidata$, $(\wt\rmX_0,\wt Y_0)$,
$\srcdata$, and $\tgtdata$, conditional on 
$\traindata$.
\end{theorem}

\begin{remark}[Exact guarantee \textit{vs.} Approximation of $\tau_{\mathrm{CP}}(\wt y_0;\wtsrg_0)$]\label{rem:mc}
Theorem~\ref{thm:cp-omega0} applies to the \textit{exact} supremum $\tau_{\mathrm{CP}}(\wt y_0;\wtsrg_0)$ in \eqref{eq:tcp-omega0}. Computations of $\tau_{\mathrm{CP}}(\wt y_0;\wtsrg_0)$ for different $\wtsrg_0$ require different optimization techniques. When using $\wtsrg$ in the experiments, we approximate $\tau_{\mathrm{CP}}(\wt y_0;\wtsrg)$ by a Monte Carlo (MC) maximum over points sampled from $\wtsrg$. This MC threshold can underestimate the exact supremum, so the MC-based prediction set \textit{does not itself inherit} the guarantee of Theorem~\ref{thm:cp-omega0}, and the reported results for $\wtsrg$ are empirical. 
\end{remark}

Additionally, the set construction is monotone in the weight region: smaller confidence regions yield smaller prediction sets at the same level. This property, together with the comparison between $\wtsrg\LP$ and $\wtsrg\GE$ in Corollary~\ref{cor:www}, is summarized in Theorem~\ref{thm:cp-monotone-omega}. The proof is provided in Appendix~\ref{sec:cp-monotone-omega}.

\begin{theorem}\label{thm:cp-monotone-omega}
If $\wtsrg_1\subset\wtsrg_2$, then $F_{\mathrm{CP}}(\wt\vx_0;\wtsrg_1)\subset F_{\mathrm{CP}}(\wt\vx_0;\wtsrg_2)$ for all $\wt\vx_0$. In particular, $F_{\mathrm{CP}}(\wt\vx_0;\wtsrg\LP)\subset F_{\mathrm{CP}}(\wt\vx_0;\wtsrg\GE)$.
\end{theorem}

\section{Experiments}\label{sec:experiments}
We evaluate the empirical performance of MaC-LP through standard benchmarks and real-world domain adaptation. Across both settings, we aim to
(i) demonstrate the \textit{tightness} of confidence regions 
for the importance weights
established in  Section \ref{sec:mac_lp}  in
comparison to the interval-based Gaussian elimination approach \citep{si2023pac}; 
(ii) assess the practical impact of MaC-LP on \textit{conformal prediction}, which serves as a representative task of MaC-LP for uncertainty quantification.

\subsection{Experimental design}
\label{sec:experimental_design}
Using standard benchmark datasets,
text data AGNews \citep{zhang2015character},
image data MNIST \citep{lecun1998gradient} and CIFAR-10
\citep{krizhevsky2009learning}, and
text-image news classification data
N24News
\citep{wang2022n24news},
we verify our theoretical guarantees while controlling potential confounding factors.

We split each dataset into two parts: one part is \textit{training
  pool}  for classifier development, the other part is
\textit{analysis set}
for simulating label shift and downstream
applications. The precise split ratios are summarized in
Table~\ref{tab:classifier_summary}.  
We train 10 classifiers for each dataset, the details are in Appendix \ref{app:benchmark_classifier}.

Next, we simulate label
shift by generating the label proportion vectors $\rvv_p$ and
$\rvv_q$ while maintaining invariant class-conditional distributions. We
define an experimental grid of $7 \times 8 = 56$ settings by varying
shift intensity and sample size. For each of  the 10 classifiers
associated with a benchmark, we perform 100 independent trials,
resulting in 1,000 
repetitions per setting to ensure statistical significance. All
samples are shuffled to eliminate ordering bias prior to analysis.

\paragraph{Label shift simulation.}
We generate the source and target label proportion vectors $\rvv_p$ and
  $\rvv_q$ from a
Dirichlet distribution. The source vector is sampled as $\rvv_p \sim
\text{Dir}(\mathbf{1}_K \cdot 10^4)$, providing a near-uniform
baseline. The target vector is sampled as $\rvv_q \sim
\text{Dir}(\mathbf{1}_K \cdot 10^a)$, where the concentration
parameter $a \in \{-3, -2, -1, 0, 1, 2, 3\}$. Smaller values of $a$ yield higher sparsity across class labels, resulting in more extreme shifts between the source and target distributions.

\paragraph{Data sampling procedure.}
For each repetition,
the synthetic datasets are generated from the analysis set using a multinomial
stratified sampling procedure. Specifically, given a label proportion
vector $\rvv$ and a sample size $c$, we sample a count vector $\rvc$
from $\MultiN(c, \rvv)$, then for each $k\in[K]$, we draw $c_k$ observations.
Calibration Set ($\calidata$): We draw a source
domain calibration set 
by the stratified sampling
approach without replacement using 
$c=5,000$ and $\rvv = \rvv_p$ from the analysis set. This calibration
set is independent of the following two datasets. 
Test Set ($\testdata$): We draw a target
domain test set 
by the stratified sampling
approach with replacement using 
$c= 2,500$ and $\rvv = \rvv_q$ from the analysis set. 
Estimation Sets ($\srcdata, \tgtdata$): We draw source domain dataset 
by the same stratified sampling
approach with replacement using 
$c$ from 1000 to 8000, step size of 1000 and $\rvv =
\rvv_p$ from the analysis set. Further, we draw target domain dataset by the same protocol except using $\rvv = \rvv_q$ from the analysis set.

\subsection{Results}\label{sec:experiments_res}
\subsubsection{Objective (i): tightness of confidence regions}
Similar to \citet{si2023pac}, we compute the $1 - \delta/(K^2+K)$
Clopper-Pearson \citep{clopper1934use} interval for every element in
$\C$ and $\q$. Here, we set $\delta = 1/5,000$ (matching $|\calidata|$). Following Section \ref{sec:mac_lp}, we
compute $\wtsrg$ and $\wtsrg\LP$, and following Appendix~\ref{sec:gaussian}, we compute $\wtsrg\GE$. 

\paragraph{LP yields shorter intervals than GE.}
We compute the marginal length of $\wtsrg\LP$ as $l_{k,\mathrm{LP}} = \ol \omega_k - \ul \omega_k$ for all $k \in [K]$, and compute the marginal length of $\wtsrg\GE$ analogously, denoted by $l_{k,\mathrm{GE}}$ for all $k \in [K]$.
 Next, we compute the $\log_2$ length ratio as $r_k = \log_2(l_{k, \mathrm{GE}}/l_{k, \mathrm{LP}})$ for all $k\in[K]$. A one-sample $t$-test was conducted on $\{r_k\}_{k=1}^K$, where the null hypothesis is that the $\log_2$ ratio is less than or equal to zero.
In the left panel of Figure \ref{fig:n24newsLV},
we plot the mean $\log_2$ length ratio, calculated as
$K^{-1}\sum_{k=1}^K r_k$,
against the statistical significance expressed as
 $-\log_{10}(p\text{-value})$
for all 56 label shift settings for N24News. 

In the left panel of Figure \ref{fig:n24newsLV}, LP consistently yields shorter marginal confidence intervals than the GE baseline. 
 The $t$-test results remain well above the
   $-\log_{10}(0.05)$ threshold, indicating that 
  the reduction in uncertainty is statistically significant for N24News. Notably, we observe the highest
  $-\log_{10}(p\text{-value})$ and length ratios, suggesting that LP is particularly effective when dealing
  with complex classifiers where joint class dependencies are more
  pronounced. 
The similar results are shown in Appendix \ref{app:lvres} for AGNews, MNIST, and CIFAR-10.

\begin{wrapfigure}{r}{0.6\textwidth}
\begin{center}
\centerline{\includegraphics[width=1.03\linewidth]{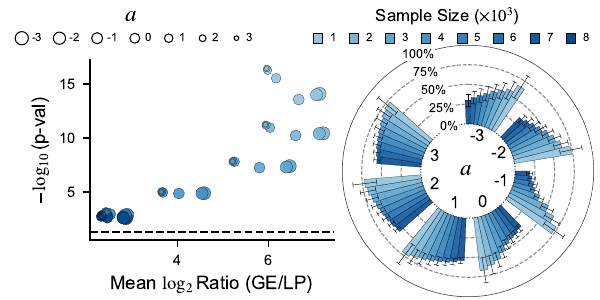}}
      \caption{N24News. (Left) Volcano plot of the interval length ratios (GE vs. LP) and their corresponding $t$-test significance levels. (Right) Volume ratio of $\wtsrg$ to $\wtsrg\LP$ across varying shift intensities $a$ and sample sizes.}
      \label{fig:n24newsLV}
    \end{center}
\end{wrapfigure}

\paragraph{MaC captures finer dependency than marginal bounds (LP).}
We use a Monte Carlo method to quantify the ratio of the volumes of $\wtsrg$ and
$\wtsrg\LP$. Specifically, we sample 10,000 points uniformly from
$\wtsrg\LP$ to represent $\vol(\wtsrg\LP)$, and then count the
number of points that satisfy the matrix constraints to represent
$\vol(\wtsrg)$.
For N24News, we plot a bar chart
of the mean ratio $\vol(\wtsrg)/\vol(\wtsrg\LP)$ over 1,000
repetitions (presented as a percentage), and the error bars represent
the standard deviation in the right panel of Figure~\ref{fig:n24newsLV}. 

As shown in the bar plots  in Figure
\ref{fig:n24newsLV}, the volume of $\wtsrg$ is
consistently a 
small fraction of the $\wtsrg\LP$ across all shift intensities and all
sample sizes. The
 region  $\wtsrg$, by considering 
dependent structure of the linear system, inherently excludes vast
regions of ``infeasible'' weight combinations that violate the
 matrix constraints. The region $\wtsrg$ maintains a tight polytope
even under extreme label shift (lower $a$). This volumetric tightness
is the primary driver for improved downstream efficiency, as it
directly restricts the search space for the weighted quantiles in the
conformal prediction phase. 
The similar results for AGNews, MNIST, and CIFAR-10 are in Appendix \ref{app:lvres}.

\subsubsection{Objective (ii): assessment of downstream task}
To assess the practical utility of MaC-LP, we apply the confidence regions $\wtsrg$ and $\wtsrg\LP$ to the weighted conformal prediction (WCP) framework \citep{podkopaev2021distribution}. For comparison, we evaluate the WCP procedure using the $\wtsrg\GE$, the Oracle vector $\bomega^*$, and the point estimate $\wh\bomega$ derived from the BBSE method \citep{lipton2018detecting}. These experiments are implemented across 56 label shift settings for each benchmark and are repeated 1,000 times to ensure statistical robustness.

For each label shift setting, we utilize an independent calibration set $\calidata$ ($m_1=5,000$) 
to compute non-conformity scores (details in Appendix \ref{app:implementationCP}). We construct prediction sets $F_{\mathrm{CP}}(\vx_t; \star)$ for the test instances in $\testdata$ by calculating weighted quantiles over the methods $\star \in \{\bomega^*, \wh\bomega, \wtsrg, \wtsrg\LP, \wtsrg\GE\}$. The confidence levels for $\wtsrg$, $\wtsrg\LP,$ and $\wtsrg\GE$ are fixed at $1-\delta$ with $\delta = 1/5,000$. This specific choice of $\delta$ ensures that weight estimation error does not compromise the target coverage level $1-\alpha$, where $\alpha = 0.1$. The thresholds for $\wtsrg\LP$ and $\wtsrg\GE$ are computed exactly, whereas the threshold for $\wtsrg$ is approximated by the maximum over 1,000 MC samples from $\wtsrg$ (Remark~\ref{rem:mc}).
We evaluate the impact of importance weights estimation on the trade-off between statistical validity and predictive efficiency using two standard metrics, Empirical Coverage and Average Set Size.

\begin{figure}[t]%{r}{0.99\textwidth}
    \begin{center}
    \vspace{12pt}
    \centerline{\includegraphics[width=\linewidth]{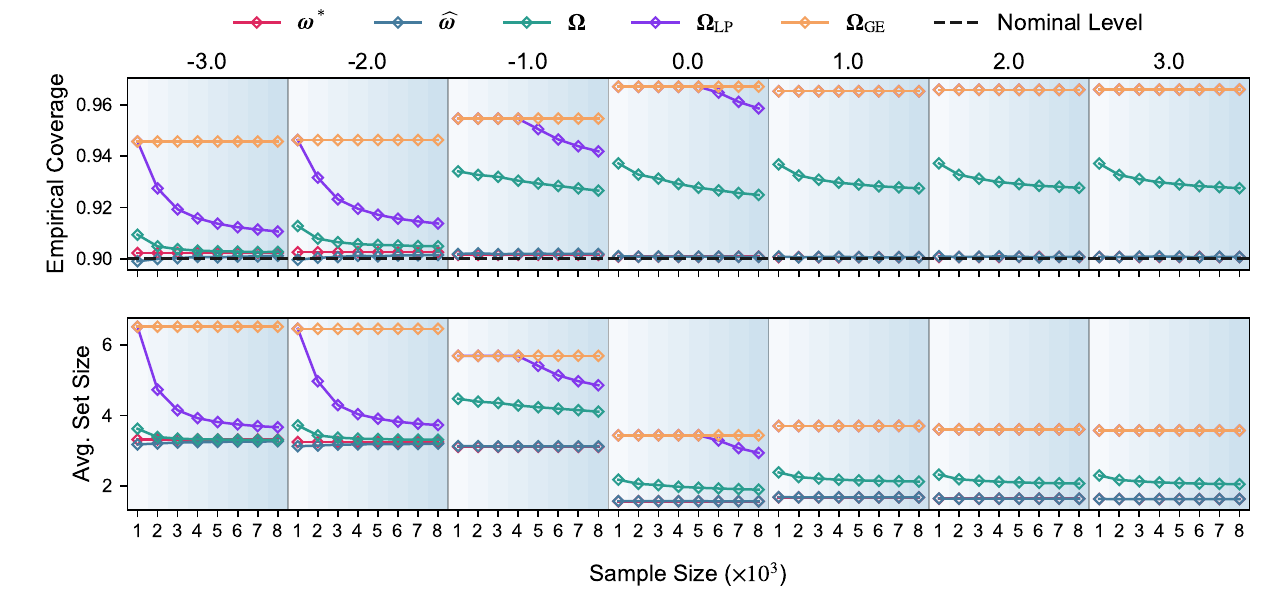}}
    \caption{Evaluation of weighted conformal prediction on N24News. The top panel displays empirical coverage across varying shift intensities ($a$) and sample sizes ($n$). The bottom panel shows average prediction set sizes.}
    \label{fig:cpn24news}
    \end{center}
    \vspace{-18pt}
\end{figure}

\paragraph{MaC enhances efficiency while maintaining reliability.} 
Figure \ref{fig:cpn24news} illustrates the evaluation metrics for the N24News benchmark. Among all methods, $\wtsrg$ consistently maintains empirical coverage above the nominal 90\% level. While $\wtsrg$ typically yields lower coverage than the overly conservative $\wtsrg\LP$ and $\wtsrg\GE$ regions, it remains empirically above the nominal level.
Notably, the point estimate $\wh\bomega$ suffers from undercoverage under extreme shifts ($a \in \{-3, -2\}$), thereby losing statistical validity.

Regarding efficiency, the Oracle specification $\bomega^*$ provides the theoretical lower bound for prediction set size. $\wtsrg$ maintains a stable trajectory that closely tracks the Oracle performance, whereas $\wtsrg\LP$ and $\wtsrg\GE$ exhibit substantial efficiency loss because their feasible regions include weight combinations prohibited by the matrix constraints. We further observe that while $\wh\bomega$ may produce smaller prediction sets in certain label shift settings, these results are unreliable as they coincide with a failure to maintain nominal coverage. Similar trends are observed across all other benchmarks, as detailed in Appendix \ref{app:cpres}.

To evaluate real-world robustness, we test MaC-LP on the nuImages dataset \citep{caesar2020nuscenes} by constructing a domain adaptation scenario, treating data from Singapore as the source domain and images captured in Boston as the target domain (details in Appendix \ref{app:nuImages_details}). This geographic shift presents a complex environment where the label shift assumption may be violated, allowing us to evaluate the resilience of MaC-LP in the presence of naturalistic distribution shifts. 

\subsubsection{Evaluation on real-world data (nuImages)}\label{sec:nuImages_res}
In Figure \ref{fig:nuImages}(a), GE still produces much wider confidence intervals for importance weights especially when sample size is small, where the uncertainty is larger. Over the all sample sizes, $\wtsrg$ is consistently smaller than $\wtsrg\LP$ as shown in Figure \ref{fig:nuImages}(b).  

For the conformal prediction task, we use the 90\% nominal prediction level. From Figure \ref{fig:nuImages}(c), we can see that even applying ``true'' importance weights, the coverage probability is still under 90\%, indicating much more complicated data structure and much more noise in the real-world dataset than that in the benchmarks. However, $\wtsrg$, $\wtsrg\LP$, and $\wtsrg\GE$ achieve coverage probability above the nominal level. In Figure \ref{fig:nuImages}(d), $\wtsrg$ provides the best balance between prespecified coverage level and the prediction set size. Notably, applying the point estimation, $\wh\bomega$, remains to have deteriorative performance. Since label shift is likely violated here, Theorem~\ref{thm:cp-omega0} does not apply and these results are empirical rather than an exact guarantee: the confidence regions also absorb part of the {label shift assumption violation}, a buffer that shrinks as the sample size grows, which is illustrated by a controlled experiment in Appendix~\ref{app:violation}.

\begin{figure}[t]
    \centering
    \includegraphics[width=.496\linewidth]{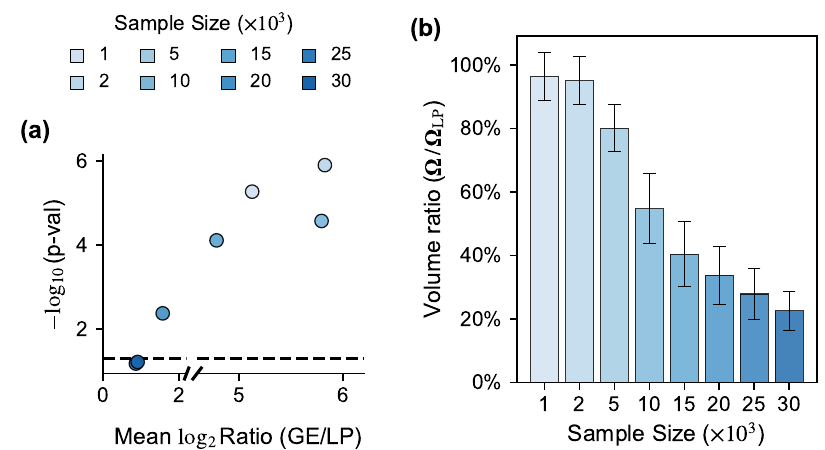}
    \includegraphics[width=.496\linewidth]{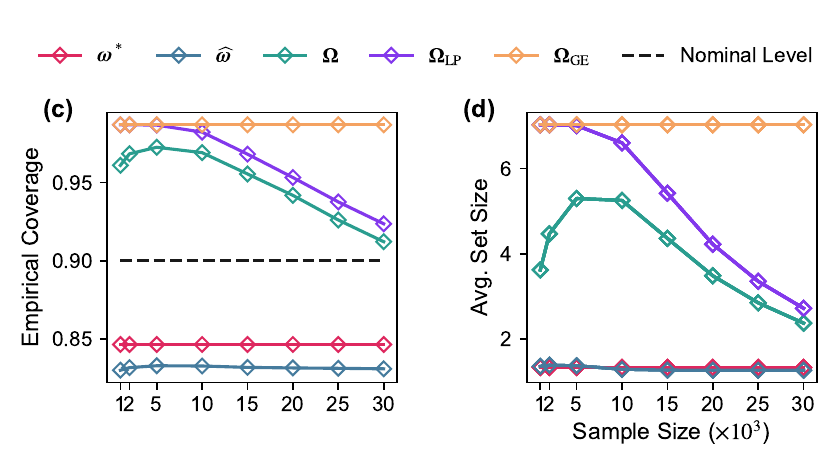}
    \caption{nuImages. (a) Volcano plot of the interval length ratios (GE vs. LP) and their corresponding $t$-test significance levels. (b) Volume ratio of $\wtsrg$ to $\wtsrg\LP$ across varying sample sizes. (c) Empirical coverage across varying sample sizes ($n$). (d) Average prediction set sizes.}
    \label{fig:nuImages}
\end{figure}

\subsection{Runtime and implementation analysis}\label{sec:runtime}
In Appendix~\ref{app:runtime}, on synthetic data with $K$ up to 500, $\wtsrg\LP$ takes 0.046 seconds at $K=100$ and 1.72 seconds at $K=500$ (0.030 seconds and 0.263 seconds for $\wtsrg\GE$), while the mean $\log_2$ length ratio of GE to LP intervals grows to about 10.9 for $K\ge200$. A near-zero target proportion gives $\omega_k^*\to0$, which is investigated in $a\in\{-3,-2\}$ settings, and the LP solver had no failures. Rare source classes and empty or unbounded regions are discussed in Appendix~\ref{app:stability}.

\section{Discussion}
In this work, we propose \textit{Matrix Constraints Linear Programming} (MaC-LP) for label shift, a novel finite-sample method for uncertainty-aware inference. Specifically, MaC captures the most refined finite-sample valid confidence regions
given the elementwise
confidence intervals of $\C^*$ and $\q^*$.
%the confusion matrix and the label proportion vector.
Our experimental results in conformal prediction highlight that MaC-LP is practically effective in addressing uncertainty quantification problems.
Furthermore, the MaC construction in Theorem \ref{thm:lp} provides an explicit convex representation of the confidence region $\wtsrg$ without requiring point estimates of importance weights. This convexity allows \text{MaC} to be integrated into domain adaptation tasks as a convex optimization problem, bypassing LP-related overhead while maintaining precision. 

As limitations, our guarantees require the exact label shift assumption, the optimality of $\wtsrg$ is relative to the input $(\ul\C, \ol\C, \ul\q, \ol\q)$, and the MC approximation of computing $\tau_{\mathrm{CP}}(\wt y_0;\wtsrg)$ does not inherit Theorem~\ref{thm:cp-omega0}.
Nevertheless, \text{MaC-LP} enriches the label shift toolbox by providing statistical guarantees. 
Further discussions regarding optimality, computational complexity, classifier robustness, application scope, and scalability for large number of classes are provided in Appendix~\ref{app:extended_discussion}. 

\paragraph{Code availability.}
The python package is available at
\url{https://github.com/carayms/maclp}.

\begin{comment}
\begin{ack}
\red
Use unnumbered first level headings for the acknowledgments. All acknowledgments
go at the end of the paper before the list of references. Moreover, you are required to declare
funding (financial activities supporting the submitted work) and competing interests (related financial activities outside the submitted work).
More information about this disclosure can be found at: \url{https://neurips.cc/Conferences/2026/PaperInformation/FundingDisclosure}.

Do {\bf not} include this section in the anonymized submission, only in the final paper. You can use the \texttt{ack} environment provided in the style file to automatically hide this section in the anonymized submission.
\end{ack}
\end{comment}

\bibliographystyle{plainnat}
\bibliography{lbsftRefs}

@inproceedings{caesar2020nuscenes,
  title={nuscenes: A multimodal dataset for autonomous driving},
  author={Caesar, Holger and Bankiti, Varun and Lang, Alex H and Vora, Sourabh and Liong, Venice Erin and Xu, Qiang and Krishnan, Anush and Pan, Yu and Baldan, Giancarlo and Beijbom, Oscar},
  booktitle={Proceedings of the IEEE/CVF conference on computer vision and pattern recognition},
  pages={11621--11631},
  year={2020}
}

@article{oettli1964compatibility,
  title={Compatibility of approximate solution of linear equations with given error bounds for coefficients and right-hand sides},
  author={Oettli, Werner and Prager, William},
  journal={Numerische Mathematik},
  volume={6},
  number={1},
  pages={405--409},
  year={1964},
  publisher={Springer}
}

@inproceedings{tian2023elsa,
	title={{ELSA}: Efficient label shift adaptation through the lens of semiparametric models},
	author={Tian, Qinglong and Zhang, Xin and Zhao, Jiwei},
	booktitle={International Conference on Machine Learning},
	pages={34120--34142},
	year={2023},
	organization={PMLR}
}

@inproceedings{lipton2018detecting,
	title={Detecting and correcting for label shift with black box predictors},
	author={Lipton, Zachary and Wang, Yu-Xiang and Smola, Alexander},
	booktitle={International conference on machine learning},
	pages={3122--3130},
	year={2018},
	organization={PMLR}
}

@article{garg2020unified,
  title={A unified view of label shift estimation},
  author={Garg, Saurabh and Wu, Yifan and Balakrishnan, Sivaraman and Lipton, Zachary},
  journal={Advances in Neural Information Processing Systems},
  volume={33},
  pages={3290--3300},
  year={2020}
}

@article{saerens2002adjusting,
	title={Adjusting the outputs of a classifier to new a priori probabilities: a simple procedure},
	author={Saerens, Marco and Latinne, Patrice and Decaestecker, Christine},
	journal={Neural computation},
	volume={14},
	number={1},
	pages={21--41},
	year={2002},
	publisher={MIT Press}
}

@article{azizzadenesheli2019regularized,
	title={Regularized learning for domain adaptation under label shifts},
	author={Azizzadenesheli, Kamyar and Liu, Anqi and Yang, Fanny and Anandkumar, Animashree},
	journal={arXiv preprint arXiv:1903.09734},
	year={2019}
}

@inproceedings{alexandari2020maximum,
	title={Maximum likelihood with bias-corrected calibration is hard-to-beat at label shift adaptation},
	author={Alexandari, Amr and Kundaje, Anshul and Shrikumar, Avanti},
	booktitle={International Conference on Machine Learning},
	pages={222--232},
	year={2020},
	organization={PMLR}
}

@article{si2023pac,
	title={{PAC} prediction sets under label shift},
	author={Si, Wenwen and Park, Sangdon and Lee, Insup and Dobriban, Edgar and Bastani, Osbert},
	journal={arXiv preprint arXiv:2310.12964},
	year={2023}
}

@article{krizhevsky2009learning,
	title={Learning multiple layers of features from tiny images},
	author={Krizhevsky, Alex and Hinton, Geoffrey and others},
	year={2009},
	publisher={Toronto, ON, Canada}
}

@article{lecun1998gradient,
	title={Gradient-based learning applied to document recognition},
	author={LeCun, Yann and Bottou, L{\'e}on and Bengio, Yoshua and Haffner, Patrick},
	journal={Proceedings of the IEEE},
	volume={86},
	number={11},
	pages={2278--2324},
	year={1998},
	publisher={Ieee}
}

@inproceedings{podkopaev2021distribution,
  title={Distribution-free uncertainty quantification for classification under label shift},
  author={Podkopaev, Aleksandr and Ramdas, Aaditya},
  booktitle={Uncertainty in artificial intelligence},
  pages={844--853},
  year={2021},
  organization={PMLR}
}

@article{tibshirani2019conformal,
  title={Conformal prediction under covariate shift},
  author={Tibshirani, Ryan J and Foygel Barber, Rina and Candes, Emmanuel and Ramdas, Aaditya},
  journal={Advances in neural information processing systems},
  volume={32},
  year={2019}
}

@inproceedings{vovk2012conditional,
  title={Conditional validity of inductive conformal predictors},
  author={Vovk, Vladimir},
  booktitle={Asian conference on machine learning},
  pages={475--490},
  year={2012},
  organization={PMLR}
}

@article{park2021pac,
  title={{PAC} prediction sets under covariate shift},
  author={Park, Sangdon and Dobriban, Edgar and Lee, Insup and Bastani, Osbert},
  journal={arXiv preprint arXiv:2106.09848},
  year={2021}
}

@article{clopper1934use,
  title={The use of confidence or fiducial limits illustrated in the case of the binomial},
  author={Clopper, Charles J and Pearson, Egon S},
  journal={Biometrika},
  volume={26},
  number={4},
  pages={404--413},
  year={1934},
  publisher={JSTOR}
}

@article{simonyan2014very,
  title={Very deep convolutional networks for large-scale image recognition},
  author={Simonyan, Karen},
  journal={arXiv preprint arXiv:1409.1556},
  year={2014}
}

@inproceedings{he2016deep,
  title={Deep residual learning for image recognition},
  author={He, Kaiming and Zhang, Xiangyu and Ren, Shaoqing and Sun, Jian},
  booktitle={Proceedings of the IEEE conference on computer vision and pattern recognition},
  pages={770--778},
  year={2016}
}

@inproceedings{sandler2018mobilenetv2,
  title={Mobilenetv2: Inverted residuals and linear bottlenecks},
  author={Sandler, Mark and Howard, Andrew and Zhu, Menglong and Zhmoginov, Andrey and Chen, Liang-Chieh},
  booktitle={Proceedings of the IEEE conference on computer vision and pattern recognition},
  pages={4510--4520},
  year={2018}
}

@inproceedings{he2016identity,
  title={Identity mappings in deep residual networks},
  author={He, Kaiming and Zhang, Xiangyu and Ren, Shaoqing and Sun, Jian},
  booktitle={Computer Vision--ECCV 2016: 14th European Conference, Amsterdam, The Netherlands, October 11--14, 2016, Proceedings, Part IV 14},
  pages={630--645},
  year={2016},
  organization={Springer}
}

@inproceedings{radosavovic2020designing,
  title={Designing network design spaces},
  author={Radosavovic, Ilija and Kosaraju, Raj Prateek and Girshick, Ross and He, Kaiming and Doll{\'a}r, Piotr},
  booktitle={Proceedings of the IEEE/CVF conference on computer vision and pattern recognition},
  pages={10428--10436},
  year={2020}
}

@article{valiant1984theory,
  title={A theory of the learnable},
  author={Valiant, Leslie G},
  journal={Communications of the ACM},
  volume={27},
  number={11},
  pages={1134--1142},
  year={1984},
  publisher={ACM New York, NY, USA}
}

@book{vovk2005algorithmic,
  title={Algorithmic learning in a random world},
  author={Vovk, Vladimir and Gammerman, Alexander and Shafer, Glenn},
  volume={29},
  year={2005},
  publisher={Springer}
}

@article{liu2019roberta,
    title = {RoBERTa: A Robustly Optimized BERT Pretraining Approach},
    author = {Yinhan Liu and Myle Ott and Naman Goyal and Jingfei Du and
              Mandar Joshi and Danqi Chen and Omer Levy and Mike Lewis and
              Luke Zettlemoyer and Veselin Stoyanov},
    journal={arXiv preprint arXiv:1907.11692},
    year = {2019},
}

@inproceedings{devlin2019bert,
  title={{BERT}: Pre-training of deep bidirectional transformers for language understanding},
  author={Devlin, Jacob and Chang, Ming-Wei and Lee, Kenton and Toutanova, Kristina},
  booktitle={Proceedings of the 2019 conference of the North American chapter of the association for computational linguistics: human language technologies, volume 1 (long and short papers)},
  pages={4171--4186},
  year={2019}
}

@article{dosovitskiy2020image,
  title={An image is worth 16x16 words: Transformers for image recognition at scale},
  author={Dosovitskiy, Alexey},
  journal={arXiv preprint arXiv:2010.11929},
  year={2020}
}

@article{zhang2015character,
  title={Character-level Convolutional Networks for Text Classification},
  author={Zhang, Xiang and Zhao, Junbo and LeCun, Yann},
  journal={Advances in Neural Information Processing Systems},
  volume={28},
  year={2015}
}

@inproceedings{wang2022n24news,
  author    = {Wang, Zhen  and  Shan, Xu  and  Zhang, Xiangxie  and  Yang, Jie},
  title     = {N24News: A New Dataset for Multimodal News Classification},
  booktitle      = {Proceedings of the Language Resources and Evaluation Conference},
  month          = {June},
  year           = {2022},
  address        = {Marseille, France},
  publisher      = {European Language Resources Association},
  pages     = {6768--6775},
  url       = {https://aclanthology.org/2022.lrec-1.729}
}

@article{wu2024online,
  title={Online feature updates improve online (generalized) label shift adaptation},
  author={Wu, Ruihan and Datta, Siddhartha and Su, Yi and Baby, Dheeraj and Wang, Yu-Xiang and Weinberger, Kilian Q},
  journal={Advances in Neural Information Processing Systems},
  volume={37},
  pages={106924--106954},
  year={2024}
}

@article{wu2025addressing,
  title={Addressing label shift in distributed learning via entropy regularization},
  author={Wu, Zhiyuan and Choi, Changkyu and Cao, Xiangcheng and Cevher, Volkan and Ramezani-Kebrya, Ali},
  journal={arXiv preprint arXiv:2502.02544},
  year={2025}
}

@article{masserano2024classification,
  title={Classification under nuisance parameters and generalized label shift in likelihood-free inference},
  author={Masserano, Luca and Shen, Alex and Doro, Michele and Dorigo, Tommaso and Izbicki, Rafael and Lee, Ann B},
  journal={arXiv preprint arXiv:2402.05330},
  year={2024}
}

@inproceedings{
xu2025heterogeneous,
title={Heterogeneous Label Shift: Theory and Algorithm},
author={Chao Xu and Xijia Tang and Chenping Hou},
booktitle={Forty-second International Conference on Machine Learning},
year={2025},
url={https://openreview.net/forum?id=BY3SSEVcjV}
}

@article{zec2024overcoming,
  title={Overcoming label shift in targeted federated learning},
  author={Zec, Edvin Listo and Breitholtz, Adam and Johansson, Fredrik D},
  journal={arXiv preprint arXiv:2411.03799},
  year={2024}
}

@article{wu2021online,
  title={Online adaptation to label distribution shift},
  author={Wu, Ruihan and Guo, Chuan and Su, Yi and Weinberger, Kilian Q},
  journal={Advances in Neural Information Processing Systems},
  volume={34},
  pages={11340--11351},
  year={2021}
}

@book{neumaier1990interval,
  title={Interval methods for systems of equations},
  author={Neumaier, Arnold},
  number={37},
  year={1990},
  publisher={Cambridge university press}
}

@book{moore2009introduction,
  title={Introduction to interval analysis},
  author={Moore, Ramon E and Kearfott, R Baker and Cloud, Michael J},
  year={2009},
  publisher={SIAM}
}

@inproceedings{plassier2023conformal,
  title={Conformal prediction for federated uncertainty quantification under label shift},
  author={Plassier, Vincent and Makni, Mehdi and Rubashevskii, Aleksandr and Moulines, Eric and Panov, Maxim},
  booktitle={International Conference on Machine Learning},
  pages={27907--27947},
  year={2023},
  organization={PMLR}
}

@article{gibbs2021adaptive,
  title={Adaptive conformal inference under distribution shift},
  author={Gibbs, Isaac and Candes, Emmanuel},
  journal={Advances in Neural Information Processing Systems},
  volume={34},
  pages={1660--1672},
  year={2021}
}

@article{blyth1983binomial,
  title={Binomial confidence intervals},
  author={Blyth, Colin R and Still, Harold A},
  journal={Journal of the American Statistical Association},
  volume={78},
  number={381},
  pages={108--116},
  year={1983},
  publisher={Taylor \& Francis}
}

@article{casella1986refining,
  title={Refining binomial confidence intervals},
  author={Casella, George},
  journal={Canadian Journal of Statistics},
  volume={14},
  number={2},
  pages={113--129},
  year={1986},
  publisher={Wiley Online Library}
}

@article{blaker2000confidence,
  title={Confidence curves and improved exact confidence intervals for discrete distributions},
  author={Blaker, Helge},
  journal={Canadian Journal of Statistics},
  volume={28},
  number={4},
  pages={783--798},
  year={2000},
  publisher={Wiley Online Library}
}

@article{brown2001interval,
  title={Interval estimation for a binomial proportion},
  author={Brown, Lawrence D and Cai, T Tony and DasGupta, Anirban},
  journal={Statistical science},
  volume={16},
  number={2},
  pages={101--133},
  year={2001},
  publisher={Institute of Mathematical Statistics}
}

@article{thulin2014cost,
  title={The cost of using exact confidence intervals for a binomial proportion},
   volume={8},
   ISSN={1935-7524},
   DOI={10.1214/14-ejs909},
   number={1},
   journal={Electronic Journal of Statistics},
   publisher={Institute of Mathematical Statistics},
   author={Thulin, Måns},
   year={2014},
   month=jan
}

%%%%%%%%%%%%%%%%%%%%%%%%%%%%%%%%%%%%%%%%%%%%%%%%%%%%%%%%%%%%%%%%%%%%%%%%%%%%%%%
%%%%%%%%%%%%%%%%%%%%%%%%%%%%%%%%%%%%%%%%%%%%%%%%%%%%%%%%%%%%%%%%%%%%%%%%%%%%%%%
% APPENDIX
%%%%%%%%%%%%%%%%%%%%%%%%%%%%%%%%%%%%%%%%%%%%%%%%%%%%%%%%%%%%%%%%%%%%%%%%%%%%%%%
%%%%%%%%%%%%%%%%%%%%%%%%%%%%%%%%%%%%%%%%%%%%%%%%%%%%%%%%%%%%%%%%%%%%%%%%%%%%%%%

\newpage
\appendix
\renewcommand{\thetable}{S\arabic{table}}
\setcounter{table}{0}
\renewcommand{\thefigure}{S\arabic{figure}}
\setcounter{figure}{0}

\section{Workflow of MaC-LP}\label{sec:workflow}
\vspace*{22pt}
\begin{figure}[ht]
    \centering
    \includegraphics[width=.992\linewidth]{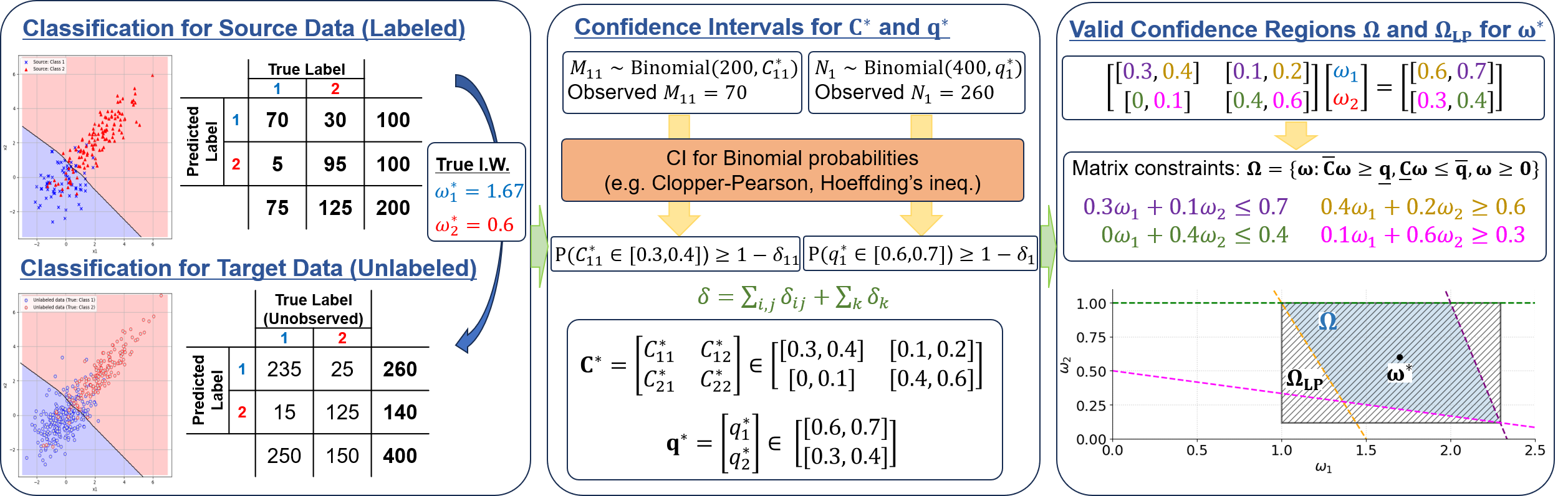}
    \caption{
    Workflow of MaC-LP in a binary classification toy example: (1) apply a classifier to both source and target datasets; (2) obtain elementwise confidence intervals for $\C^*$ and $\q^*$; (3) construct confidence regions $\wtsrg\FS$ and $\wtsrg\LP$. We can observe that $\wtsrg\FS$ and $\wtsrg\LP$ capture the true importance weigths $\bomega^*$.}
    \label{fig:workflow}
\end{figure}
\vspace*{25pt}

\section{Extended discussion}
\label{app:extended_discussion}
\paragraph{Scope of optimality: simplex and multinomial constraints.}
The optimality of $\wtsrg$ holds given the elementwise confidence intervals $\crgn$ and $\qrgn$, and it does not exploit the facts that $\sum_{i,j}C_{ij}^*=1$, $\sum_k q_k^*=1$, or that $\rmM$ and $\rmN$ are jointly multinomial. Imposing such constraints on $(\C,\q)$ can make the region smaller or equal (Figure \ref{fig:simplex}). However, because $\bomega=\C^{-1}\q$ is a nonlinear map, the resulting region need not be convex and need not admit a linear characterization as in Theorem~\ref{thm:lp}: the row-by-row construction in the proof of Theorem~\ref{thm:lp} is no longer available once the rows of $\C$ are coupled. Figure~\ref{fig:simplex} illustrates this on the two-class example of Figure~\ref{fig:workflow}: the weights generated by admissible pairs satisfying the simplex constraints form a strict, non-polytope subset of $\wtsrg$. Characterizing and computing such regions exactly is an interesting open problem; we view the construction in Theorem \ref{thm:lp} as a balance between tightness and mathematical and computational tractability.

\begin{figure}[!ht]
    \centering
    \vspace{1.6\baselineskip }
    \includegraphics[width=.7\linewidth]{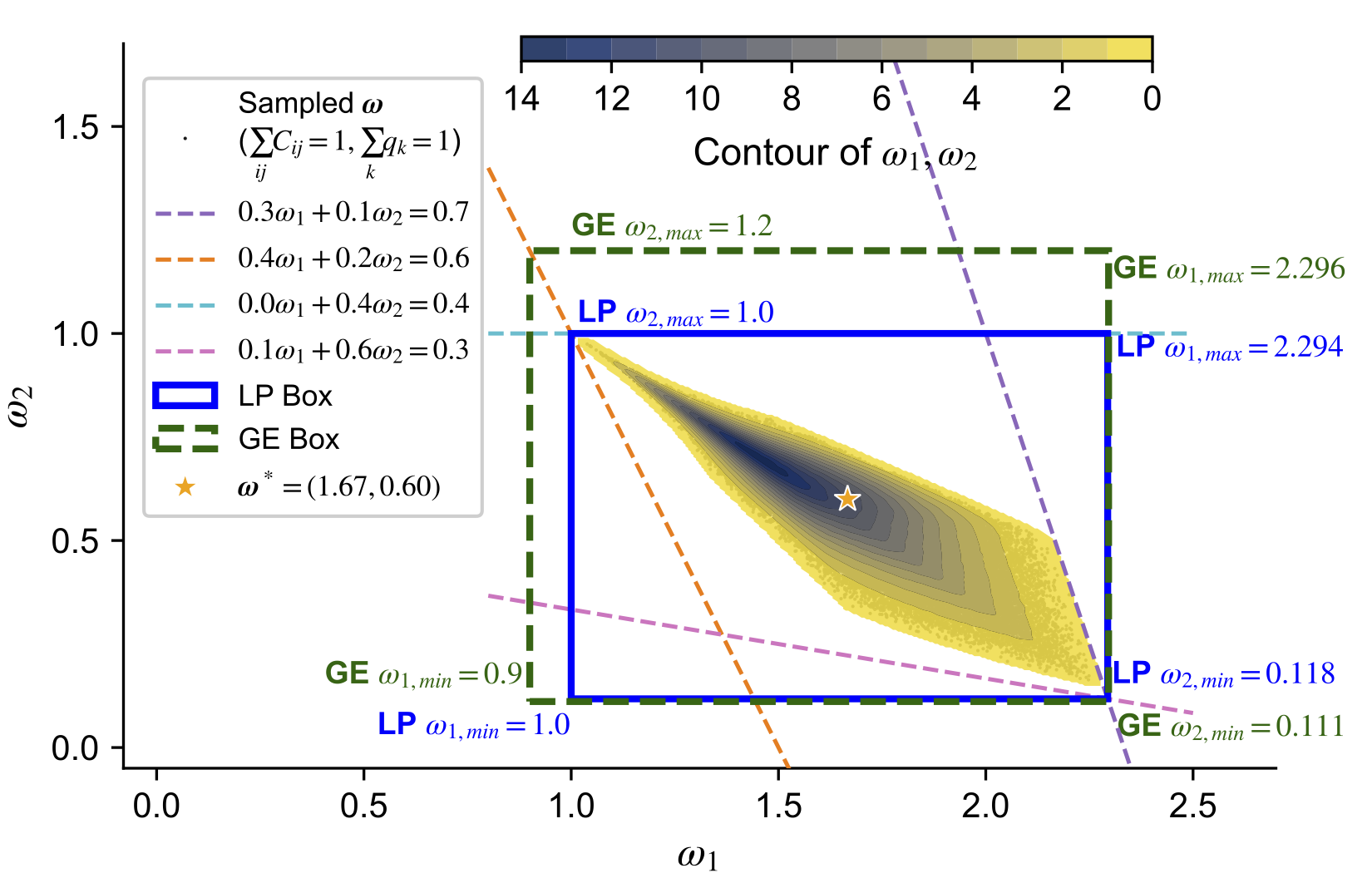}
    \caption{Two-class example of Figure~\ref{fig:workflow}. The dashed lines are the four constraints of $\wtsrg$ in Theorem~\ref{thm:lp}; the blue and green boxes are $\wtsrg\LP$ and $\wtsrg\GE$. The shaded points are the weights $\bomega$ generated by admissible pairs $(\C,\q)\in\crgn\times\qrgn$ that additionally satisfy $\sum_{i,j}C_{ij}=1$ and $\sum_k q_k=1$, with their density contours.}
    \vspace{1.3\baselineskip }
    \label{fig:simplex}
\end{figure}

\paragraph{Computational complexity and parallelization for $\wtsrg\LP$.}
While the theoretical complexity for solving the linear program for $\wtsrg\LP$ is $\mathcal{O}(K^4)$, where $K$ is the number of classes, this does not represent a practical bottleneck for standard classification tasks because the optimization process is highly amenable to parallel computing. Furthermore, the convexity of $\wtsrg$ ensures that it can be handled by a wide range of efficient off-the-shelf solvers. Wall-clock runtimes for $K$ up to 500 are reported in Appendix~\ref{app:runtime}.

\paragraph{Robustness to classifier quality and calibration.}
A significant advantage of \text{MaC-LP} is its resilience to the performance of the underlying black-box classifier. As illustrated in Figure~\ref{fig:ece} and Table \ref{tab:acc}, our framework maintains superior performance regardless of whether the base classifiers are well-calibrated or possess high accuracy. This decoupling suggests that \text{MaC-LP} effectively extracts importance weight information even from suboptimal classifiers, providing a reliable framework in practical scenarios where perfect calibration is unattainable.

\paragraph{Coverage guarantees \textit{vs.} point predictions.}
Standard domain adaptation (DA) techniques primarily focus on improving point prediction accuracy, often neglecting the quantification of predictive uncertainty. In contrast, our work centers on providing $1-\alpha$ coverage guarantees. Even when the predictive model is limited in capacity, our constructed confidence regions ensure that the true label is contained within the set with prespecified probability. This statistical safety guarantee provides more information of the $Y|\rmX$ distribution in the target domain than simple expectation-based estimates, $\mathbb{E}_Q(Y|\rmX)$.
Accordingly, the \textit{downstream efficiency} claimed in this paper refers to set-valued inference, i.e., smaller prediction sets at a guaranteed coverage level.
 When considering prediction accuracy in the target domain, 
the confidence region does not automatically improves point-prediction metrics; turning $\wtsrg$ into a point prediction requires an additional criterion over $\wtsrg$ (e.g., a minimax rule $\min_{\hat y}\max_{\bomega\in\wtsrg}R(\hat y;\bomega)$ for a weighted risk), to which the coverage guarantee of $\wtsrg$ does not transfer. We leave this direction to future work.

\paragraph{Comparison with asymptotic and bootstrap uncertainty quantification.}
A direct empirical comparison with other label shift estimators requires each of them to come with an implementable, data-dependent confidence region for $\bomega^*$. Most existing estimators are developed for point estimation. BBSE \citep{lipton2018detecting} and RLLS \citep{azizzadenesheli2019regularized} provide high-probability error bounds, but these contain unspecified constants and cannot be evaluated as confidence regions from the observed data. ELSA \citep{tian2023elsa} establishes asymptotic normality with an explicit limiting covariance, but does not provide a feasible covariance estimator or an implementation of the corresponding region. A nonparametric bootstrap could produce approximate intervals, but it approximates both the finite-sample problem and its solution, requires hundreds to thousands of refits, and leaves open how to turn bootstrap draws into a multi-dimensional region of prescribed level (e.g., a minimum-volume region would require multivariate density estimation). In all cases the result is an asymptotic or approximate procedure rather than the exact finite-sample guarantee studied here. We regard a systematic comparison with such asymptotic and bootstrap-based regions as future work.

\paragraph{Scalability to large class dimensions.}
For tasks involving a massive number of labels, the whole model scalability may have negative impact on tractability. In such instances, we recommend a hierarchical approach to decompose the label space, for example, incorporating with a decision tree model. While we treat the optimization of these hierarchies as an extension for future study, the current conclusion of \text{MaC-LP} remains highly effective for the experimented datasets.

\section{Broader impacts} \label{app:impact}
This work establishes a mathematically rigorous framework to achieve operational efficiency under principled safety bounds in machine learning. However, as a purely theoretical framework, MaC-LP has not been verified in complex real-world environments thoroughly. We caution practitioners that these theoretical guarantees may not automatically hold in unverified practical deployments. Comprehensive empirical testing and safety validations are strictly required before applying this method to high-stakes applications.

\newpage
\section{Details of related work}\label{sec:related}
\paragraph{Estimation for importance weights under label shift.}\label{sec:review-est}
Although our work focuses on constructing finite-sample confidence regions for the importance weights, we briefly review existing methods for point estimation of $\bomega^*$ under label shift.
Recall that we are given finite samples $\srcdata = \{(\vx_s, y_s)\}_{s=1}^{m}$ and $\tgtdata = \{\vx_{m+t}\}_{t=1}^{n}$ from source and target domain, respectively. 
In the source domain, for $i,j\in[K]$, let
$M_{ij}=\sum_{s=1}^m I\{g(\vx_s)=i,\ y_s=j\}$,
where $I\{\cdot\}$ is the indicator function. 
Denote $\rmM = (M_{ij})_{i,j\in[K]}$ which is  a $K\times K$ matrix. Note that $\sum_{i=1}^K\sum_{j=1}^K M_{ij}=m$.
Then
$\vect(\rmM)\,|\, \C^*\sim \MultiN\{m,\vect(\C^*)\}$,
where $\vect$ denotes the vectorization of a matrix.
In the target domain, for $k\in[K]$, let
$N_k=\sum_{t=1}^n I\{g(\vx_{m+t})=k\}$,
and denote $\rmN=(N_1,\ldots,N_K)^\top$.
Similarly, we have $\sum_{k=1}^K N_k=n$ and $\rmN\,|\, \q^*\sim \MultiN(n,\q^*)$.
Motivated by the identity $\C^*\bomega^*=\q^*$ derived in Section~\ref{sec:confusion-matrix-approach}, a natural approach to estimate $\bomega^*$ is via plug-in solutions.
When $\C^*$ is invertible, solving the linear system
$\C^* \bomega^* = \q^*$ yields a valid estimator of the importance weights.
In practice, since $\C^*$ and $\q^*$ are unknown, one replaces them by their sample analogues
$\wh\C=(\wh C_{ij})_{i,j\in[K]}$ and $\wh\q=(\wh q_1,\ldots,\wh q_K)^\top$,
where $\wh C_{ij}=M_{ij}/m$ and $\wh q_k=N_k/n$.
The resulting estimator $\wh\bomega_{\mathrm{BBSE}}=\wh\C^{-1}\wh\q$ is referred to as black box shift estimation (BBSE) \citep{lipton2018detecting}.

Regularized learning of label shift (RLLS) \citep{azizzadenesheli2019regularized} modifies BBSE by stabilizing the plug-in solution via regularization.
Specifically, RLLS computes $\wh\bomega_{\mathrm{RLLS}}$ as a minimizer of
\bse
\wh\bomega_{\mathrm{RLLS}}
=\argmin_{\bomega}\ \|\wh\C\bomega-\wh\q\|_2^2 + \lambda \|\bomega-\1_K\|_2^2,
\ese
where $\1_K$ is the length-$K$ vector of ones.
The regularizer $\|\bomega-\1_K\|_2^2$ shrinks $\bomega$ toward $\1_K$, corresponding to no label shift.
In some implementations of BBSE and RLLS, additional constraints such as $\bomega\ge 0$ and a normalization condition (ensuring that the implied target label distribution sums to one) may also be imposed.

More generally, moment-based methods can be formulated by matching reweighted source moments to target moments.
Efficient label shift adaptation (ELSA) \citep{tian2023elsa} estimates $\bomega^*$ by solving the moment-matching equation
\bse
\frac{1}{m}\sum_{s=1}^m h(\vx_s)\wh\omega_{\mathrm{ELSA},y_s}
=
\frac{1}{n}\sum_{t=1}^n h(\vx_{m+t}),
\ese
equivalently,
\bse
\sum_{j=1}^K\frac{1}{m}\sum_{s=1}^m h(\vx_s)\1\{y_s=j\}\wh\omega_{\mathrm{ELSA},j}
=
\frac{1}{n}\sum_{t=1}^n h(\vx_{m+t}).
\ese
Moreover, ELSA constructs an optimal choice of the moment function $h$ (within an appropriate class) so that the resulting estimator $\wh\bomega_{\mathrm{ELSA}}$ attains the semiparametric efficiency bound for $\bomega^*$ under label shift.

Complementary to these moment-based approaches, likelihood-based methods estimate the target label distribution directly under the label shift model.
Maximum likelihood label shift (MLLS) uses a well-calibrated soft-max classifier to obtain an estimate $\wh p(y)$ of the source label distribution and then estimates the target label distribution $\wh q(y)$ by maximum likelihood, typically via an EM algorithm \citep{alexandari2020maximum}.
The resulting estimator for $\bomega^*$ is $\wh\bomega_{\mathrm{MLLS}}=(\wh\omega_{\mathrm{MLLS},1},\ldots,\wh\omega_{\mathrm{MLLS},K})^\top$, where $\wh\omega_{\mathrm{MLLS},y}=\wh q(y)/\wh p(y)$ for $y\in[K]$.

\paragraph{Recent works on label shift in extended settings.}
The label shift setting continues to receive active attention.
In particular, recent work incorporates label shift into online, federated/distributed, and generalized modeling settings.
Recent work studies online label shift, where unlabeled target samples arrive sequentially and the learner must update adaptation or calibration procedures over time \citep{wu2021online, wu2024online}.
Label shift has also been considered in federated learning, where clients may exhibit heterogeneous label distributions and the goal is to improve performance in the target domain \citep{zec2024overcoming}.
Label shift has further been studied in multi-node distributed learning, where each node may have access to both source and target data under heterogeneous label distributions \citep{wu2025addressing}; for example, \citet{wu2025addressing} use an entropy-based regularizer to stabilize importance-weight estimation across nodes.
\citet{masserano2024classification} introduce a generalized label shift model of the form $p(\vx\mid y,\nu)=q(\vx\mid y,\nu)$, where $\nu$ denotes nuisance parameters that affect the class-conditional distribution.
Label shift has also been studied under structured heterogeneity, including settings where the source and target feature representations differ (e.g., cross-lingual text) but the label-shift structure is leveraged for adaptation \citep{xu2025heterogeneous}.
When these settings can be reduced to (or approximated by) the standard label-shift identity $\C^*\bomega^*=\q^*$, our framework provides finite-sample confidence regions for $\bomega^*$ as an alternative to point estimation in such pipelines.

\paragraph{Uncertainty propagation through linear systems.}
A common approach to finite-sample uncertainty quantification for parameters defined by linear identities is to treat the estimated coefficients and right-hand side as interval-valued objects and propagate these intervals through interval linear system solvers.
In this literature, interval Gaussian elimination is a standard procedure for bounding the solution set of $\rmA\vx=\vb$ when $(\rmA,\vb)$ are only known within uncertainty sets, and related alternatives include Krawczyk method \citep{neumaier1990interval, moore2009introduction}.
However, interval Gaussian elimination repeatedly applies subtraction and division that reverse endpoints, which yields overly conservative bounds for $\vx$ \citep{moore2009introduction}.
In the label shift setting, \citet{si2023pac} adopt this interval Gaussian elimination procedure by propagating binomial confidence intervals through the updates of the procedure to obtain confidence intervals for importance weights; details are provided in Appendix~\ref{sec:gaussian}.
These developments motivate propagating uncertainty by optimizing directly over the underlying constraints, rather than through interval Gaussian elimination or direct inversion.

\paragraph{Finite-sample prediction under distribution shift.}
Finite-sample prediction sets are a common downstream objective for importance weights under distribution shift, since uncertainty-aware decisions often require coverage guarantees that hold in the target domain.

Conformal prediction provides a general route to constructing prediction sets with finite-sample coverage guarantees \citep{vovk2005algorithmic}, and a growing literature studies how to modify conformal procedures when the test distribution differs from the training or calibration distribution.
For example, \citet{tibshirani2019conformal} develop weighted conformal prediction for covariate shift by incorporating importance weights into the conformal calibration step, while subsequent work considers related shift structures such as label shift and associated weighting-based corrections \citep{podkopaev2021distribution, plassier2023conformal}, as well as adaptive conformal inference that updates calibration online under sequential distribution shift \citep{gibbs2021adaptive}.

Complementary to conformal prediction, probably approximately correct (PAC) prediction sets provide coverage guarantees that hold with high probability conditional on the calibration sample, and have been extended to covariate shift and label shift settings by explicitly incorporating importance weights to ensure target-domain reliability \citep{park2021pac, si2023pac}.
These connections emphasize that tight, finite-sample-valid uncertainty quantification for importance weights can have a direct impact on the informativeness of prediction sets under distribution shift.

\subsection{The connection to Oettli–Prager theorem} \label{sec:connect_OPT}
The Oettli–Prager theorem \citep{oettli1964compatibility} provides strong theoretical support for our formulation, and the theorem validates that set-based linear systems under interval uncertainty can be exactly characterized by closed-form linear inequalities, avoiding the existential search over parameters ($\exists \C, \q$). 

However, Theorem \ref{thm:lp} differs fundamentally in its \textit{problem formulation} and \textit{inequality construction}:

\begin{itemize}[leftmargin=.5cm]
\item
Writing $\C_c=(\ul\C+\ol\C)/2$, $\Delta=(\ol\C-\ul\C)/2$, $\q_c=(\ul\q+\ol\q)/2$ and $\vd=(\ol\q-\ul\q)/2$, the Oettli–Prager theorem assumes nominal midpoints $(\C_c, \q_c)$ perturbed by tolerances $(\Delta, \vd)$. In contrast, there are no prespecified midpoints $(\C_c, \q_c)$ in our setting; the boundaries $[\ul\C, \ol\C]$ and $[\ul\q, \ol\q]$ are the \textit{primary}, data-driven objects obtained directly from statistical inequality bounds. In addition, defining midpoints $(\C_c, \q_c)$ would be a purely artificial construction rather than a reflection of the underlying data-generating process. For example, the Clopper–Pearson interval in our setting is defined through Beta quantiles, which are rarely symmetric about the point estimate, so a midpoint has no natural meaning.
\item
In one dimension coordinate, two intervals are overlapped if $(x_{\text{start}} - y_{\text{end}})(x_{\text{end}} - y_{\text{start}}) \le 0$, for $[x_{\text{start}}, x_{\text{end}}]$ and $[y_{\text{start}}, y_{\text{end}}]$.
Inspired by this, we characterized $\wtsrg$ by finding overlaps between $[\ul\C\bomega,\ol\C\bomega]$ and $[\ul\q,\ol\q]$. The proof of Theorem \ref{thm:lp} uses the intermediate value argument to establish the overlap, whereas Oettli–Prager theorem is proved by determining whether a hyperplane meets a box.
\end{itemize}

Thus, while the Oettli–Prager theorem supports the mathematical principle from an algebraic perturbation perspective with nominal centers, Theorem \ref{thm:lp} adopts an endpoint-first geometric formulation directly tailored to statistical bound settings.

\section{Confidence intervals for binomial probability}
\label{sec:ci-binomial}
We briefly review standard methods for constructing $(1-\gamma)$ confidence intervals for binomial probabilities.
Let $B\sim \mathrm{Binomial}(n,\pi^*)$.
An exact $(1-\gamma)$ confidence interval for $\pi^*$ is obtained by inverting level-$\gamma$ binomial tests, yielding the Clopper--Pearson interval:
\bse
\ul \pi
&=&\inf\Big\{\pi\in[0,1]: \pr\big(\mathrm{Binomial}(n,\pi)\ge B\big)\le \gamma_1\Big\},\\
\ol \pi
&=&\sup\Big\{\pi\in[0,1]: \pr\big(\mathrm{Binomial}(n,\pi)\le B\big)\le \gamma_2\Big\},
\ese
where $\gamma_{1}+\gamma_{2}=\gamma$.
This yields an exact interval $[\ul \pi,\ol \pi]$ satisfying $\pr\{\pi^*\in [\ul \pi,\ol \pi]\}\ge 1-\gamma$.
Equivalently, letting $Q_{\alpha}(X)$ denote the $\alpha$-quantile of a random variable $X$,
\bse
\ul\pi&=&Q_{\gamma_{1}}\{\mathrm{Beta}(B,n-B+1)\},\\
\ol\pi&=&Q_{1-\gamma_2}\{\mathrm{Beta}(B+1,n-B)\}.
\ese
When $1\le B\le n-1$, it is conventional to choose $\gamma_{1}=\gamma_{2}=\gamma/2$.
When $B=0$, one sets $\gamma_{1}=0$ and $\gamma_{2}=\gamma$, and when $B=n$, one sets $\gamma_{1}=\gamma$ and $\gamma_{2}=0$.

Beyond the Clopper--Pearson interval, several exact refinements aim to reduce conservativeness while maintaining guaranteed coverage.
The Blyth--Still--Casella interval \citep{blyth1983binomial, casella1986refining} chooses an optimal tail split $(\gamma_{1},\gamma_{2})$ with $\gamma_{1}+\gamma_{2}=\gamma$ to yield a shorter confidence interval.
The Blaker interval \citep{blaker2000confidence} constructs an exact interval using an acceptability function and is often shorter than Clopper--Pearson; in some cases, the resulting set can be a union of disjoint intervals.

There are also $(1-\gamma)$ confidence intervals based on concentration inequalities for the sample proportion $\wh\pi=B/n$.
These bounds yield algebraically simple intervals, though they may be conservative in finite samples.
First, since $\Var(\wh\pi)=\pi^*(1-\pi^*)/n\le 1/(4n)$, Chebyshev's inequality gives, for any $t>0$,
\bse
\pr(|\wh\pi-\pi^*|\ge t)\le \frac{\Var(\wh\pi)}{t^2}
\le \frac{1}{4n t^2}.
\ese
Thus taking $t=1/\sqrt{4n\gamma}$ yields the interval
\bse
\pi^*\in[\wh\pi-1/\sqrt{4n\gamma},\ \wh\pi+1/\sqrt{4n\gamma}]\cap[0,1]
\ese
with probability at least $1-\gamma$.

Next, since $\wh\pi$ is the average of $n$ i.i.d.\ Bernoulli$(\pi^*)$ variables, Hoeffding's inequality gives, for any $t>0$,
\bse
\pr(|\wh\pi-\pi^*|\ge t)\le 2\exp(-2nt^2).
\ese
Setting $2\exp(-2nt^2)=\gamma$ yields the explicit $(1-\gamma)$ interval
\bse
\pi^*\in\Big[\wh\pi-\sqrt{\frac{\log(2/\gamma)}{2n}},\ \wh\pi+\sqrt{\frac{\log(2/\gamma)}{2n}}\Big]\cap[0,1].
\ese

Lastly, since $B$ is a sum of $n$ Bernoulli random variables, which are bounded,
Bernstein's inequality gives a sharper confidence interval than Hoeffding's inequality with a variance term in the bound. To be specific, for any $t>0$,
\bse
\pr(|\wh\pi-\pi^*|\ge t)\le
2\exp\!\Big\{\!-\frac{nt^2}{2\pi^*(1-\pi^*)+2t/3}\Big\},
\ese
One can further obtain a closed-form bound by upper bounding $\pi^*(1-\pi^*)\le 1/4$, giving
\bse
\pr\{|\wh\pi-\pi^*|\ge t\}\le
2\exp\!\Big(\!-\frac{nt^2}{1/2+2t/3}\Big),
\ese
and then solving for $t$ such that the right-hand side equals $\gamma$.
For comprehensive studies of interval estimation for binomial proportions, see \citet{brown2001interval} and \citet{thulin2014cost}.

\section{Proofs} \label{sec:proofs}
\subsection{Proof of \texorpdfstring{$\C^*\bomega^*=\q^*$}{Proof of C*omega*=q*}}
\label{sec:Comegaq}
\begin{proof}
For each $i\in[K]$, we obtain that
\bse
\sum_{j=1}^K C_{ij}^*\omega_j^*&=& \sum_{j=1}^K \pr_{(\rmX,Y)\sim P}\{g(\rmX)=i, Y=j\}\,
\frac{\pr_{(\rmX,Y)\sim Q}(Y=j)}{\pr_{(\rmX,Y)\sim P}(Y=j)}\\
&=& \sum_{j=1}^K \pr_{(\rmX,Y)\sim P}\{g(\rmX)=i \mid Y=j\}\,\pr_{(\rmX,Y)\sim Q}(Y=j)\\
&=& \sum_{j=1}^K \pr_{(\rmX,Y)\sim Q}\{g(\rmX)=i \mid Y=j\}\,\pr_{(\rmX,Y)\sim Q}(Y=j)\\
&=& \pr_{(\rmX,Y)\sim Q}\{g(\rmX)=i\}
= q_i^*,
\ese
which implies $\C^*\bomega^*=\q^*$.
\end{proof}
\subsection{Proof of Proposition~\ref{prop:unionlaw}}\label{sec:unionlaw}
\begin{proof}
By the union bound, we obtain that
\bse
\pr(\bomega^* \not\in \wtsrg_0) &\le& \pr(\C^* \not \in \crgn \textrm{ or }\q^* \not \in \qrgn)\\
&\le& \sum_{i, j \in[K]} \pr(C_{ij}^* \not \in [\ul C_{ij}, \ol C_{ij}]) + \sum_{k \in[K]} \pr(q_{k}^* \not \in [\ul q_k, \ol q_k])\\
&\le& \sum_{i, j \in[K]} \delta_{ij} + \sum_{k \in[K]} \delta_{k} = \delta.
\ese 
\end{proof}

\subsection{Proof of Theorem~\ref{thm:lp}}\label{sec:lp}

\begin{proof}
Suppose that $\C \in \crgn$, $\q \in \qrgn$, and that $\bomega\in\wtsrg_0$ satisfies $\C\bomega=\q$ and
$\bomega\ge0$. The linear equation $\C\bomega=\q$ is equivalent to
\bse
\sum_{j=1}^K C_{ij}\bomega_j = q_i.
\ese
Since $\omega_j\ge0$, we get that
\bse
\sum_{j=1}^K \ol C_{ij}\omega_j \ge \sum_{j=1}^K C_{ij}\omega_j = q_i \ge \ul q_i,\\
\sum_{j=1}^K \ul C_{ij}\omega_j \le \sum_{j=1}^K C_{ij}\omega_j = q_i \le \ol q_i,
\ese
which implies that $\bomega \in \wtsrg\FS$.

Now, we prove the other direction. Suppose that $\bomega \in \wtsrg\FS$. Then for each $i\in[K]$,
we can apply the following procedure. If
$\sum_{j=1}^K \ol C_{ij}\omega_j \le \ol q_i$ for any $i\in[K]$,
take $C_i^\top = (\ol C_{i1},\ldots,\ol C_{iK})$ and
$q_i=\sum_{j=1}^K \ol C_{ij}\omega_j\in [\ul q_i,\ol q_i]$. Otherwise, assume $\sum_{j=1}^K \ol C_{ij}\omega_j > \ol q_i$. For $l=0,1,...,K$, define
\bse
q_i(l)=\sum_{j=1}^l \ul C_{ij}\omega_j + \sum_{j=l+1}^K \ol C_{ij}\omega_j.
\ese
Then $q_i(l)$ is a decreasing function of $l$. Also,
$q_i(K)=\sum_{j=1}^K \ul C_{ij}\omega_j \le \ol q_i$ by the condition, and $q_i(0)=\sum_{j=1}^K \ol C_{ij}\omega_j > \ol q_i$ by the assumption. Thus, we can find $l_0$
such that $q_i(l_0-1)> \ol q_i \ge q_i(l_0)$. 
Note that $\omega_{l_0}> 0$ since $0<q_i(l_0-1)- q_i(l_0)=(\ol C_{i,l_0}-\ul C_{i,l_0})\omega_{l_0}$.
Let
$C_i^\top =
(\ol C_{i1},\ldots,\ol C_{i,l_0-1},\wt C_{i,l_0},\ul C_{i,l_0+1},\ldots,\ul C_{iK})$
and 
$q_i=\ol q_i$,
where
\bse
\wt C_{i,l_0}=\ul C_{i,l_0}+\frac{\ol q_i-q_i(l_0)}{\omega_{l_0}}.
\ese
Then $\wt C_{i,l_0}\in[\ul C_{i,l_0},\ol C_{i,l_0}]$ and $C_i^\top\bomega=q_i$.
Taking $C_i^\top$ as the $i$-th row of $\C$ and $q_i$ as the $i$-th element of $\q$, we get that
$\C\bomega=\q$, where $\C\in\crgn$ and $\q\in\qrgn$, meaning $\bomega\in\wtsrg_0$.
\end{proof}

\subsection{Proof of Corollary~\ref{cor:www}}\label{sec:www}
\begin{proof}
The first part, $\wtsrg\FS \subset \wtsrg\LP$, is trivial from its definition. For the second part,
since $\wtsrg\GE$ exists and contains all $\C^{-1}\q$ for $\C\in\crgn$ and $\q\in\qrgn$ by its construction, we have $\wtsrg\FS \subset \wtsrg\GE$ . Since $\wtsrg\LP$ is the smallest
hyperrectangle that contains $\wtsrg\FS$, we get $\wtsrg\LP \subset \wtsrg\GE$.
\end{proof}

\subsection{Proof of Theorem \ref{thm:diameter}}\label{sec:diameter}
\begin{proof}
We prove the $\ell_\infty$ bound; the $\ell_2$ bound follows by the
same argument with $\|\cdot\|_\infty$ replaced by $\|\cdot\|_2$. 

Let $\bomega=\rmA^{-1}\q$ and $\bomega'=\rmB^{-1}\q'$ in $\wtsrg\FS$,
where $\ul\C\le\rmA,\rmB\le\ol\C$ and $\ul\q\le\q,\q'\le\ol\q$. Write
$\rmA=\ul\C+\Delta$ and $\rmB=\ul\C+\Delta'$ with
$\0\le\Delta,\Delta'\le\ol\C-\ul\C$, and write $\q=\ol\q-\vd$ and
$\q'=\ol\q-\vd'$ with $\0\le\vd,\vd'\le\ol\q-\ul\q$. Then
\bse
\bomega-\bomega'&=&\rmA^{-1}(\ol\q-\vd)-\rmB^{-1}(\ol\q-\vd')\\
&=&(\rmA^{-1}-\rmB^{-1})(\ol\q-\vd')-\rmA^{-1}(\vd-\vd')\\
&=&\rmA^{-1}(\Delta'-\Delta)\rmB^{-1}(\ol\q-\vd')-\rmA^{-1}(\vd-\vd'),
\ese
where the last line holds by $\rmA^{-1}-\rmB^{-1}=\rmA^{-1}(\rmB-\rmA)\rmB^{-1}$.
Taking $\|\cdot\|_\infty$ gives
\bse
\|\bomega-\bomega'\|_\infty\le\|\rmA^{-1}\|_\infty\|\Delta'-\Delta\|_\infty\|\rmB^{-1}\|_\infty\|\ol\q-\vd'\|_\infty+\|\rmA^{-1}\|_\infty\|\vd-\vd'\|_\infty.
\ese

Since $0\le\Delta,\Delta'\le\ol\C-\ul\C$ and $0\le\vd,\vd'\le\ol\q-\ul\q$, we have
$\|\Delta'-\Delta\|_\infty\le\|\ol\C-\ul\C\|_\infty$, $\|\vd-\vd'\|_\infty\le\|\ol\q-\ul\q\|_\infty$, and $\|\ol\q-\vd'\|_\infty\le\|\ol\q\|_\infty$.
Moreover, for any $\Delta$ with $\|\Delta\|_\infty\le\|\ol\C-\ul\C\|_\infty$,
$(\ul\C+\Delta)^{-1}=(I+\ul\C^{-1}\Delta)^{-1}\ul\C^{-1}$,
and
\bse
\|(I+\ul\C^{-1}\Delta)^{-1}\|_\infty&=&\sup_{\vv\neq
  0}\frac{\|(I+\ul\C^{-1}\Delta)^{-1}\vv\|_\infty}{\|\vv\|_\infty}
=\frac{1}{\inf_{\|\vw\|_\infty=1}\|(I+\ul\C^{-1}\Delta)\vw\|_\infty}\\
&\le&
\frac{1}{1-\sup_{\|\vw\|_\infty=1}\|\ul\C^{-1}\Delta\vw\|_\infty}
\le\frac{1}{1-\|\ul\C^{-1}\|_\infty\|\Delta\|_\infty},
\ese
where we wrote $\vv=(I+\ul\C^{-1}\Delta)\vw$ and
used $\|\C^{-1}\Delta\|_\infty<1$.
Hence, 
\bse
\|(\ul\C+\Delta)^{-1}\|_\infty\le\frac{\|\ul\C^{-1}\|_\infty}{1-\|\ul\C^{-1}\|_\infty\|\Delta\|_\infty}\le\frac{\|\ul\C^{-1}\|_\infty}{1-\|\ul\C^{-1}\|_\infty\|\ol\C-\ul\C\|_\infty}.
\ese
Applying this bound to $\rmA^{-1}$ and $\rmB^{-1}$ and substituting into the previous inequality yields the stated ${\mathrm{diam}}_\infty(\wtsrg\FS)$ bound. The ${\mathrm{diam}}_2(\wtsrg\FS)$ bound follows identically by replacing $\|\cdot\|_\infty$ with $\|\cdot\|_2$ throughout.
\end{proof}

\subsection{Proof of Theorem~\ref{thm:cp-omega0}}\label{sec:cp-omega0}
\begin{proof}
By Theorem~2 of \citet{podkopaev2021distribution}, we have
\bse
\mathbb{P}_{(\wt\rmX_0,\wt Y_0)\sim Q}\{\wt Y_0 \in F_{\mathrm{CP}}(\wt\rmX_0;\bomega^*)\} \ge 1-\alpha.
\ese
Also, if $\bomega^* \in \wtsrg_0$, we get $F_{\mathrm{CP}}(\wt\rmX_0;\bomega^*) \subset F_{\mathrm{CP}}(\wt\rmX_0;\wtsrg_0)$. Then
\bse
\mathbb{P}_{(\wt\rmX_0,\wt Y_0)\sim Q}\{\wt Y_0 \notin F_{\mathrm{CP}}(\wt\rmX_0;\wtsrg_0)\}
&=& \mathbb{P}_{(\wt\rmX_0,\wt Y_0)\sim Q}\{\bomega^* \in \wtsrg_0 \text{ and } \wt Y_0 \notin F_{\mathrm{CP}}(\wt\rmX_0;\wtsrg_0)\} \\
&& + \mathbb{P}_{(\wt\rmX_0,\wt Y_0)\sim Q}\{\bomega^* \notin \wtsrg_0 \text{ and } \wt Y_0 \notin F_{\mathrm{CP}}(\wt\rmX_0;\wtsrg_0)\} \\
&\le& \mathbb{P}_{(\wt\rmX_0,\wt Y_0)\sim Q}\{\wt Y_0 \notin F_{\mathrm{CP}}(\wt\rmX_0;\bomega^*)\} + \mathbb{P}(\bomega^* \notin \wtsrg_0) \\
&\le& \alpha + \delta.
\ese
\end{proof}

\subsection{Proof of Theorem~\ref{thm:cp-monotone-omega}}\label{sec:cp-monotone-omega}
\begin{proof}
First, \eqref{eq:tcp-omega0} implies $\tau_{\mathrm{CP}}(\wt y;\wtsrg_1)\le \tau_{\mathrm{CP}}(\wt y;\wtsrg_2)$
for all $\wt y$. Then \eqref{eq:fcp-omega0} gives the result.
\end{proof}

\section{An explicit bound obtained by Theorem \ref{thm:diameter}} \label{app:diameter}
\textbf{Theorem \ref{thm:diameter}}
\textit{Let $\ul\C\in\realNum^{K\times K}$ be invertible.
If $\|\ul\C^{-1}\|_\infty\|\ol\C-\ul\C\|_\infty<1$, then 
\bse
{\mathrm{diam}}_\infty(\wtsrg\FS)\le\frac{\|\ul\C^{-1}\|_\infty^2\|\ol\C-\ul\C\|_\infty\|\ol\q\|_\infty}{\big(1-\|\ul\C^{-1}\|_\infty\|\ol\C-\ul\C\|_\infty\big)^2}
+\frac{\|\ul\C^{-1}\|_\infty\|\ol\q-\ul\q\|_\infty}{1-\|\ul\C^{-1}\|_\infty\|\ol\C-\ul\C\|_\infty}.
\ese
If $\|\ul\C^{-1}\|_2\|\ol\C-\ul\C\|_2<1$, then 
\bse
{\mathrm{diam}}_2(\wtsrg\FS)\le\frac{\|\ul\C^{-1}\|_2^2\|\ol\C-\ul\C\|_2\|\ol\q\|_2}{\big(1-\|\ul\C^{-1}\|_2\|\ol\C-\ul\C\|_2\big)^2}
+\frac{\|\ul\C^{-1}\|_2\|\ol\q-\ul\q\|_2}{1-\|\ul\C^{-1}\|_2\|\ol\C-\ul\C\|_2}.
\ese
}

We use Hoeffding's inequality method construct confidence intervals for binomial probability, illustrating the usage of Theorem \ref{thm:diameter}.
For preparation,
let $\delta'=\delta/(K^2+K)$, and for each $i,j,k\in[K]$, let $[\ul C_{ij},\ol C_{ij}]$ and $[\ul q_k,\ol q_k]$ be $(1-\delta')$-level confidence intervals constructed with Hoeffding's inequality (Appendix~\ref{sec:ci-binomial}): if $B\sim \mathrm{Binomial}(r,\pi)$, then
\bse
\Pr\{|B/r-\pi|\le \sqrt{\frac{\log(2/\delta')}{2r}}\}\ge 1-\delta',
\ese
which gives us
\be
\ol C_{ij}-\ul C_{ij}\le \sqrt{2\log(2/\delta')/m},\qquad \ol q_k-\ul q_k\le \sqrt{2\log(2/\delta')/n}.\label{eq:hoeffding}
\ee
Assume $\ul\C$ is invertible and $\sigma_{\min}(\ul\C)\ge 1/(LK)$  for
some constant $L>0$. Then $\|\ul\C^{-1}\|_2\le LK$, and
subsequently $\|\ul\C^{-1}\|_\infty\le LK^{3/2}$.

Now we provide the $\|\cdot\|_\infty$ norm bound. By \eqref{eq:hoeffding},
\bse
\|\ol\C-\ul\C\|_\infty\le K\sqrt{2\log(2/\delta')/m},\qquad
\|\ol\q-\ul\q\|_\infty&\le& \sqrt{2\log(2/\delta')/n}.
\ese
Since $\ol q_k\le 1$ for all $k\in[K]$, we have $\|\ol\q\|_\infty\le 1$.
Assume $m$ is sufficiently large so that
$\|\ul\C^{-1}\|_\infty\|\ol\C-\ul\C\|_\infty \le LK^{5/2}\sqrt{2\log(2/\delta')/m}<1/2$.
Then by Theorem \ref{thm:diameter},
\bse
{\mathrm{diam}}_\infty(\wtsrg\FS)
&\le&\frac{L^2K^4\sqrt{2\log(2/\delta')/m}}{\big(1-LK^{5/2}\sqrt{2\log(2/\delta')/m}\big)^2}+\frac{LK^{3/2}\sqrt{2\log(2/\delta')/n}}{1-LK^{5/2}\sqrt{2\log(2/\delta')/m}}\\
&\le&4L^2m^{-1/2}K^4\sqrt{2\log(4K^2/\delta)}+2Ln^{-1/2}K^{3/2}\sqrt{2\log(4K^2/\delta)}.
\ese 

In a similar manner, we get the $\|\cdot\|_2$ norm bound under mild conditions (Appendix~\ref{sec:L2norm}):
when $LK^{2}\sqrt{2\log(2/\delta')/m}<1/2$, we have 
\bse
{\mathrm{diam}}_2(\wtsrg\FS)
\le4L^2m^{-1/2}K^{7/2}\sqrt{2\log(4K^2/\delta)}+2Ln^{-1/2}K^{3/2}\sqrt{2\log(4K^2/\delta)}.
\ese
Overall, we can observe that both ${\mathrm{diam}}_\infty(\wtsrg\FS)$ and ${\mathrm{diam}}_2(\wtsrg\FS)$ have $O\{(m^{-1/2}K^4+n^{-1/2}K^{3/2})\sqrt{\log(K/\delta)}\}$ and $O\{(m^{-1/2}K^{7/2}+n^{-1/2}K^{3/2})\sqrt{\log(K/\delta)}\}$ rate, respectively.

\subsection{Proof of $\|\cdot\|_2$ bound}\label{sec:L2norm}
\begin{proof}
By \eqref{eq:hoeffding}, we have
\bse
\|\ol\C-\ul\C\|_2\le K\sqrt{\frac{2\log(2/\delta')}{m}},\qquad \|\ol\q-\ul\q\|_2\le \sqrt{K}\sqrt{\frac{2\log(2/\delta')}{n}}.
\ese
Since $\ol q_k\le 1$  for $k\in[K]$, we have $\|\ol\q\|_2\le \sqrt K$.
Assume in addition that
$m$ is sufficiently large so that
$\|\ul\C^{-1}\|_2\|\ol\C-\ul\C\|_2 \le LK^{2}\sqrt{2\log(2/\delta')/m}<1/2$.
Then by Theorem \ref{thm:diameter},
\bse
{\mathrm{diam}}_2(\wtsrg\FS)
&\le&\frac{L^2K^{7/2}\sqrt{\frac{2\log(2/\delta')}{m}}}{\big(1-LK^{2}\sqrt{\frac{2\log(2/\delta')}{m}}\big)^2}
+\frac{LK^{3/2}\sqrt{\frac{2\log(2/\delta')}{n}}}{1-LK^{2}\sqrt{\frac{2\log(2/\delta')}{m}}}\\
&\le&4L^2m^{-1/2}K^{7/2}\sqrt{2\log(4K^2/\delta)}
+2Ln^{-1/2}K^{3/2}\sqrt{2\log(4K^2/\delta)}.
\ese
\end{proof}

\section{Details of PAC prediction}\label{app:PAC}
Let the calibration set $\calidata=\{(\wt\vx_i,\wt y_i)\}_{i=1}^{m_1}$ follow the source distribution and denote by $r(\wt\vx,\wt y)$ the nonconformity score trained on an independent training set.
The PAC prediction set $F_{\textrm{PAC}}(\wt\vx;\bomega,{\calidata})$ under label shift \citep{vovk2012conditional, park2021pac, si2023pac} is defined by
\bse
\pr_{\calidata\sim P^{m_1}}[\pr_{(\wt\rmX_0,\wt Y_0)\sim Q}\{\wt Y_0\in F_{\textrm{PAC}}
(\wt\rmX_0;\bomega,\calidata)\}\ge1-\epsilon]\ge1-\eta.
\ese

\citet{si2023pac} constructed a PAC prediction set based on the confidence region $\wtsrg\GE$. For a fixed candidate weight vector $\bomega$, calibration set $\calidata$, auxiliary random variable $V\sim\mathrm{Uniform}([0,1])^{m_1}$, and a large constant $b\ge \max_{k\in[K]}\omega_k$, their construction yields a prediction set $F_{\textrm{PAC}}(\cdot;\wtsrg_0,\calidata,V,b)$ satisfying the following PAC guarantee:
\be\label{eq:pac'}
\pr_{\calidata\sim P^{m_1},V}[\pr_{(\wt\rmX_0,\wt Y_0)\sim Q}\{\wt Y_0\in F_{\textrm{PAC}}(\wt\rmX_0;\bomega,\calidata,V,b)\}\ge1-\epsilon]\ge1-\eta,
\ee
Then, the PAC prediction set $F_{\textrm{PAC}}(\wt\vx;\bomega,{\calidata},V,b)$ is in the form of
\bse
F_{\textrm{PAC}}({\wt\vx};\bomega,\calidata,V,b)={\{\wt y\in[K]:}r(\wt\vx,\wt y)\le\tau_{\textrm{PAC}}\{T(\bomega,\calidata,V,b)\}{\}},
\ese
where $T(\bomega,\calidata,V,b)=\{(\wt\vx_i,\wt y_i)\in \calidata:V_i\le\omega_{\wt y_i}/b\}$ is a target sample generated by rejection-sampling from $\calidata$.
Let $m_0=|T(\bomega,\calidata,V,b)|$.
Here, $\tau_{\textrm{PAC}}\{T(\bomega,{\calidata},V,b)\}$ is chosen to satisfy
\bse
&&\sum_{(\wt\vx_i,\wt y_i)\in T(\bomega,\calidata,V,b)}
I\{\wt y_i\not\in F_{\textrm{PAC}}(\wt\vx_i;\bomega,\calidata,V,b)\}\\
&&=\sum_{(\wt\vx_i,\wt y_i)\in T(\bomega,\calidata,V,b)}
I[r(\wt\vx_i,\wt y_i)>\tau_{\textrm{PAC}}\{T(\bomega,\calidata,V,b)\}]\le k(m_0,\epsilon,\eta),
\ese
where
\bse
k(m_0,\epsilon,{\eta})=\max\{k:F_{\textrm{Binomial}(m_0,\epsilon)}(k)\le\eta\}.
\ese
That is, $\tau_{\textrm{PAC}}\{T(\bomega,\calidata,V,b)\}$ is the $k(m_0,\epsilon,\eta)$-th largest value of $\{r(\wt\vx_i,\wt y_i):(\wt\vx_i,\wt y_i)\in T(\bomega,\calidata,V,b)\}$.
Note that $F_{\textrm{Binomial}({m_0},\epsilon)}(\cdot)$ is the CDF of $\textrm{Binomial}({m_0},\epsilon)$.
If the true importance weight $\bomega^*$ replaces $\bomega$, then the modified PAC condition \eqref{eq:pac'} is satisfied.
When the confidence region $\wtsrg_0$ with $\pr(\bomega^*\in\wtsrg_0)\ge1-\delta$ is provided, we can define
\be\label{eq:tau_pac}
\tau_{\textrm{PAC}}\{T(\wtsrg_0,\calidata,V,b)\}=\sup_{{\bomega}\in\wtsrg_0}\tau_{\textrm{PAC}}\{T({\bomega},\calidata,V,b)\}
\ee
and
\be\label{eq:F_pac}
F_{\textrm{PAC}}(\wt\vx;\wtsrg_0,\calidata,V,b)={\{\wt y\in[K]:}r(\wt\vx,\wt y)\le\tau_{\textrm{PAC}}\{T(\wtsrg_0,\calidata,V,b)\}{\}}.
\ee
Then $F_{\textrm{PAC}}(\wt\vx;\wtsrg_0,\calidata,V,b)$ satisfies the modified PAC condition \eqref{eq:pac'} with {$1-\eta$ replaced by $1-\eta-\delta$}.

As in Section~\ref{sec:conformal-prediction}, the resulting prediction set is monotone in the weight region. Therefore, replacing $\wtsrg\GE$ by the tighter region $\wtsrg\LP$ yields a smaller prediction set while preserving the same PAC guarantee \citep{park2021pac}.

\begin{theorem}\label{th:PACww}
If $\wtsrg_1\subset\wtsrg_2$, then $F_{\mathrm{PAC}}(\wt\vx_0;\wtsrg_1,\calidata,V,b)\subset F_{\mathrm{PAC}}(\wt\vx_0;\wtsrg_2,\calidata,V,b)$ for all $\wt\vx_0$. In particular, $F_{\mathrm{PAC}}(\wt\vx_0;\wtsrg\LP,\calidata,V,b)\subset F_{\mathrm{PAC}}(\wt\vx_0;\wtsrg\GE,\calidata,V,b)$.
\end{theorem}
\begin{proof}
First, \eqref{eq:tau_pac} implies $\tau_{\mathrm{PAC}}\{T(\wtsrg_1,\calidata,V,b)\}\le \tau_{\mathrm{PAC}}\{T(\wtsrg_2,\calidata,V,b)\}$
for all $\wt y$. Then \eqref{eq:F_pac} gives the result.
\end{proof}

\begin{theorem}\label{th:pac}
Suppose that $\pr(\bomega^*\in\wtsrg_0)\ge1-\delta$.
Then
\bse
\pr_{\calidata\sim P^{m_1},V}[\pr_{(\wt\rmX_0,\wt Y_0)\sim Q}\{\wt Y_0\in F_{\textrm{PAC}}(\wt\rmX_0;\wtsrg_0,\calidata,V,b)\}\ge1-\epsilon]\ge1-\eta-\delta.
\ese
\end{theorem}
\begin{proof}
The proof follows from Theorem~3 of \cite{si2023pac}.
\end{proof}

\section{Gaussian elimination with intervals} \label{sec:gaussian}
 Given that $\C^* \in \crgn = [\ul\C, \ol\C]$ and $\q^* \in \qrgn =[\ul\q, \ol\q]$, \citet{si2023pac} introduce an intuitive way, which they named Gaussian elimination with intervals, of finding $\wtsrg\FS$ that contains $\bomega^* = \C^{*-1} \q^*$. 
 Suppose that $\ul C_{ij} \ge 0$, $\ul q_i > 0$, and $\omega_i^* > 0$ for $i,j \in [K]$. 
They follow two phases of Gaussian elimination when solving a system of linear equations $\C^* \bomega^* = \q^*$ and derive the classwise interval for $\omega_i^*$. 
First, set $\ul C_{ij}^{1} = \ul C_{ij}$, $\ol C_{ij}^{1} = \ol C_{ij}$, $\ul q_{i}^{1} = \ul q_{i}$, and $\ol q_{i}^{1} = \ol q_{i}$. In the first phase (forward elimination), the elementary row operations are applied sequentially for $k=1,...,K-1$ to delete the $(i,k)$ element in the matrix for $i>k$ by adding the multiple of the $k$-th row. 
 Then the lower bound $\ul C_{ij}^{k+1}$ and the upper bound $\ol C_{ij}^{k+1}$ are derived from the interval $[\ul \C^{k}, \ol \C^{k}]$ at the $k$-th step as
\bse
\ul C_{ij}^{k+1} = 
\begin{cases}
    0, & \hbox{if } i>k, j \le k,\\
    \ul C_{ij}^k - \frac{\ol C_{ik}^k\ol C_{kj}^k}{\ul C_{kk}^k}, & \hbox{if } i,j > k,\\
    \ul C_{ij}^k, & \hbox{otherwise.}\\
\end{cases}
\ese
\bse
\ol C_{ij}^{k+1} = 
\begin{cases}
    0, & \hbox{if } i>k, j \le k,\\
    \ol C_{ij}^k - \frac{\ul C_{ik}^k\ul C_{kj}^k}{\ol C_{kk}^k}, & \hbox{if } i,j > k,\\
    \ol C_{ij}^k, & \hbox{otherwise.}\\
\end{cases}
\ese
Simultaneously, $\ul q_i^{k+1}$ and $\ol q_i^{k+1}$ are obtained from the same row operations to be
\bse
\ul q_{i}^{k+1} = 
\begin{cases}
    \ul q_{i}^k - \frac{\ol C_{ik}^k\ol q_{k}^k}{\ul C_{kk}^k}, & \hbox{if } i > k,\\
    \ul q_{i}^k, & \hbox{otherwise.}\\
\end{cases}
\ese
\bse
\ol q_{i}^{k+1} = 
\begin{cases}
    \ol q_{i}^k - \frac{\ul C_{ik}^k\ul q_{k}^k}{\ol C_{kk}^k}, & \hbox{if } i > k,\\
    \ol q_{i}^k, & \hbox{otherwise.}\\
\end{cases}
\ese
Then $C_{ij}^{*,k+1}$ and $q_{i}^{*,k+1}$, which would have been obtained in the forward elimination step solving $\C^* \bomega^* = \q^*$, always lie in $[\ul C_{ij}^{k+1}, \ol C_{ij}^{k+1}]$ and $[\ul q_{i}^{k+1}, \ol q_{i}^{k+1}]$.
In the second phase (back substitution), they compute $\ul \omega_i$ and $\ol \omega_i$, iteratively for $i=K,...,1$, replacing the truth with intervals as in the first phase.
\bse
 \ul s_i = \sum_{j = i+1}^K \ul C_{ij}^{K} \ul \omega_j &\hbox{ and }& \ol s_i = \sum_{j = i+1}^K \ol C_{ij}^{K} \ol \omega_j ,\\
\ul \omega_i = \frac{\ul q_i - \ol s_i}{\ol C_{ii}^{K}} &\hbox{ and }&\ol \omega_i = \frac{\ol q_i - \ul s_i}{\ul C_{ii}^{K}}.
\ese
Then $\wtsrg\GE$ is defined as the $K$-dimensional hypercube  $\prod_{i=1}^K [\ul \omega_i, \ol \omega_i]$.
\citet{si2023pac} provide a theoretical result that their method yields $\omega^* \in \wtsrg\GE$ if $\ul C_{ij}^k \ge 0$, $\ul C_{ii}^k > 0$, and $\ul q_i^k \ge 0$ for all $i,j,k \in [K]$.
The basic assumption in order to satisfy the condition is that $\ol C_{ik}\ll \ul C_{kk}$. This is ensured when the classifier $g(X)$ is accurate, that is, when the diagonal terms $C_{kk}^{*}$ in $\C^*$ dominate non-diagonal terms. If the assumption is violated, we may encounter a possibility that $C_{kk}^{*,k} \approx 0$, which may lead to $\ul C_{kk}^{k} \le 0$. Then in the forward elimination phase, $\ul C_{ij}^{k+1}$ for all $i,j>k$ will be $-\infty$, which may make the algorithm impractical.
Furthermore, if $\ul q_i \le \ol s_i$ or $\ul C_{ii}^K \le 0$ for some $i$, then the back substitution phase would lead to $\ul \omega_i \le 0$ or $\ol \omega_i = \infty$, which does not provide any information about the interval of $\omega_i^*$.
 In order to deal with the nonpositive bounds, they mention that choosing a wider margin, which would, however, make $\wtsrg\GE$ larger than its optimal size.
 
 \subsection{An example to illustrate why GE is conservative}\label{sec:ge_example}
 Interval Gaussian elimination enlarges the region because for
some $(i,j)$, both $\overline{C}_{ij}$ and $\underline{C}_{ij}$ are used within the same
computation. LP never does this, since each constraint uses one consistent choice of
endpoints.

For example, apply interval Gaussian elimination to the matrix equation in Appendix~\ref{sec:workflow}, where
$C_{11}\in[0.3,0.4]$, $C_{12}\in[0.1,0.2]$, $C_{21}\in[0,0.1]$, $C_{22}\in[0.4,0.6]$,
$q_{1}\in[0.6,0.7]$, and $q_{2}\in[0.3,0.4]$. Then
$\underline{\omega}_{2}=\underline{q}^{(2)}_{2}\big/\overline{C}^{\,(2)}_{22}$, where
\bse
  \overline{C}^{(2)}_{22}
  = \overline{C}_{22}-\frac{\underline{C}_{21}\,\underline{C}_{12}}{\overline{C}_{11}}
  = 0.6-\frac{0\cdot 0.1}{0.4}
  = 0.6,
  \qquad
  \underline{q}^{(2)}_{2}
  = \underline{q}_{2}-\frac{\overline{C}_{21}\,\overline{q}_{1}}{\underline{C}_{11}}
  = 0.3-\frac{0.1\cdot 0.7}{0.3}
  = \frac{1}{15},
\ese
giving $\underline{\omega}_{2}=1/9\approx 0.111$. Here $C_{21}$ enters at its upper endpoint
in $\underline{q}^{(2)}_{2}$ and at its lower endpoint in $\overline{C}^{(2)}_{22}$, and
$C_{11}$ likewise appears as $\underline{C}_{11}$ in one place and $\overline{C}_{11}$ in the
other, so no admissible matrix attains this bound.

On the other hand, $\underline{\omega}_{2}$ in $\wtsrg$ from MaC is attained by an
admissible pair $(\mathbf{C},\mathbf{q})\in\mathcal{C}\times\mathcal{Q}$ satisfying
\bse
  \begin{pmatrix}
    \underline{C}_{11} & \underline{C}_{12}\\
    \overline{C}_{21}  & \overline{C}_{22}
  \end{pmatrix}
  \begin{pmatrix}\omega_{1}\\ \omega_{2}\end{pmatrix}
  =
  \begin{pmatrix}\overline{q}_{1}\\ \underline{q}_{2}\end{pmatrix},
\ese
which yields $\underline{\omega}_{2}=2/17\approx 0.118$, a tighter lower bound. This value is
attained at $\bm{\omega}=(2.294,\,0.118)^{\top}\ge\mathbf{0}$ with
$\mathbf{C}=\bigl(\begin{smallmatrix}0.3 & 0.1\\ 0.1 & 0.6\end{smallmatrix}\bigr)$ and
$\mathbf{q}=(0.7,\,0.3)^{\top}$, as guaranteed by the sufficiency direction of Theorem~\ref{thm:lp}. The resulting boxes are shown in Figure~\ref{fig:simplex}.

\newpage
\section{Detailed experimental setup}
\label{app:experiment_details}
The experiment was implemented using \texttt{Python 3.13} on MacBook Pro with the Apple M3 Pro chip, which has a 11-core CPU (5 performance cores and 6 efficiency cores), an 14-core GPU, and 18\,GB of unified memory. The classifiers were all trained using Google Colab with  NVIDIA H100 GPU.

\subsection{Benchmark data organization and classifiers}
\label{app:benchmark_classifier}

\setlength{\intextsep}{11pt}
\begin{table}[H]
  \caption{Summary of dataset splits and architectures. 
  \textsuperscript{*}VGG16, ResNet18, PreActResNet18, MobileNetV2, and RegNetX.}
  \label{tab:classifier_summary}
  \begin{center}
      %\resizebox{0.97\linewidth}{!}{
\begin{tabular}{@{}llcc@{}}
\toprule
Dataset & Architecture & \makecell{Training Pool\\(Train + Test)} & Analysis Set \\
 \midrule
AGNews   & RoBERTa   & 60,000 + 7,600   & 60,000 \\
MNIST    & 4-Layer CNN    & 30,000 + 10,000  & 30,000 \\
CIFAR-10 & Multiple\textsuperscript{*} & 25,000 + 10,000  & 25,000 \\
N24News  & BERT + ViT     & 24,487 + 6,122   & 30,609 \\ \bottomrule
\end{tabular}%}
  \end{center}
\end{table}

To obtain a classifier $g$,
we did the following. For AGNews, we fine-tuned 10 RoBERTa \citep{liu2019roberta} models using distinct random seeds. For MNIST,  we trained 10 four-layer CNNs with varying initializations. For CIFAR-10, we utilized five distinct architectures, VGG16 \citep{simonyan2014very}, ResNet18 \citep{he2016deep}, PreActResNet18 \citep{he2016identity}, MobileNetV2 \citep{sandler2018mobilenetv2}, and RegNetX \citep{radosavovic2020designing}, each trained with two different seeds. For N24News, we employed a multi-modal late-fusion architecture combining BERT \citep{devlin2019bert} and ViT \citep{dosovitskiy2020image}, training 10 models with different seeds.

\paragraph{Text (AGNews).} 
We created a stratified analysis set of 60,000 training examples (15,000 per class) using a fixed random seed. The remaining 60,000 examples were used for training. We fine-tuned \texttt{RoBERTa} for 3 epochs. Hyperparameters: learning rate $3\times10^{-5}$, batch sizes of 32 (train) and 64 (eval), AdamW optimizer (0.01 weight decay), and 500 warmup steps. Mixed-precision (fp16) was enabled.

\paragraph{Vision benchmarks (MNIST \& CIFAR-10).} 
For both datasets, we use the half of  total  training samples with the 10,000 standard test samples. 
\begin{itemize}
    \item \textbf{MNIST}: We trained 10 classifiers (10 epochs each) using a CNN with two convolutional layers ($5\times5$ kernels, 10 and 20 filters) and two fully connected layers (50 hidden units). The remaining 30,000 original training observations were reserved as the analysis set.
    \item \textbf{CIFAR-10}: We trained 10 classifiers (two instances each of \texttt{VGG},  \texttt{ResNet18},  \texttt{PreActResNet18} \texttt{MobileNetV2}, and \texttt{RegNetX} ). Each model was trained for 200 epochs. The remaining 25,000 training samples served as the analysis set.
\end{itemize}

\paragraph{Multimodal (N24News).} 
We implemented a late-fusion model concatenating 768-d features from \texttt{BERT} and \texttt{ViT} followed by a linear layer to 24 classes. Using a 50/50 stratified split ($\approx$ 30,609 samples each), models were trained for 3 epochs (lr $3\times10^{-5}$, batch size 16) across 10 replications. 

\subsection{Real-world data (nuImages) organization and classifiers}
\label{app:nuImages_details}
\subsubsection{Dataset and preprocessing}
We evaluate the proposed MaC-LP framework using the \textbf{nuImages}\footnote{Licenses: CC BY-NC-SA 4.0} dataset \cite{caesar2020nuscenes}, a large-scale autonomous driving dataset featuring high-resolution images ($1600 \times 900$ pixels) captured across diverse urban environments, which have 23 annotated classes. Our study utilizes the full metadata of the nuImages v1.0 release, encompassing a total of 50,786 training images and 13,394 validation(test; in our context for training a classifier) images from Singapore, alongside 16,493 training images and 3,051 validation images from Boston.

To transition from raw street-scene images to a structured classification task, we perform object-centric cropping. For each annotated instance, we extract the image content within the ground-truth bounding box. To ensure the reliability of the classifier $g(x)$ and mitigate label noise from low-resolution distant objects, we apply a visibility filter: only crops with a large resolution (larger than $64 \times 64$ pixels) are retained. Following extraction, all crops are resized to $224 \times 224$ pixels using bicubic interpolation and normalized according to ImageNet statistics to facilitate the use of pre-trained deep convolutional neural networks.

\subsubsection{Data partitioning, classifiers, and setting}
Consistent with the methodology required for the MaC-LP framework, we partition the Singapore and Boston data into two functionally distinct subsets:

\begin{enumerate}
    \item \textbf{Training Pool (train + test):} We combine the original training and validation splits from the nuImages metadata. From this unified pool, we randomly sample \textbf{$50\%$} of the Singapore train data (stratified by category) with the whole Singapore validation dataset and fine-tune the classification model $g(x)$. This phase optimizes the model to capture the feature representations of the source domain.
    
    \item \textbf{Analysis Set:} The remaining \textbf{$50\%$} of the Singapore train data and the whole Boston data serve as the Analysis Set. As specified in the MaC-LP framework, this set remains unseen during the training phase. It is utilized exclusively for label shift estimation and  downstream application tasks. 

    \item \textbf{Classifier:} We used 10 different random seeds fine-tuned ImageNet-pretrained backbone for our classifier $g(x)$ directly on the long-tailed source distribution without artificial rebalancing solely on Singapore data. The classifier was tuned for 5 epochs given the robust feature representations provided by the ImageNet-pretrained V1 model.
\end{enumerate}

The resulting data distribution for the Training Pool and Analysis sets is summarized in Table~\ref{tab:train_stats} after filtering the classes which have less than 100 samples. 
This results in 15 classes retained in our real-world data application. In Table~\ref{tab:val_stats}, only Singapore data are used in training classifiers process. We denote the number of samples in the source domain as \textbf{$m$} and in the target domain as \textbf{$n$}. This stratified splitting strategy ensures that our estimation of $\C$ is unbiased, as it is derived from data independent of the training process, thereby preserving the mathematical integrity of the linear programming bounds.

In our experimental design, \textit{Singapore} serves as the source domain and \textit{Boston} as the target domain.   The ``true" label proportions in both domain and the importance weights (Ratio) are in Table~\ref{tab:train_stats}.

\newpage
\begin{table}[H]
\centering
\caption{Statistics of the original Training Set of nuImages (larger than $64 \times 64$ pixels). Singapore and Boston represent the Source and Target domains, respectively.}
\label{tab:train_stats}
\hspace*{-.5cm}
\begin{tabular}{lrrrcr}
\toprule
 & \multicolumn{1}{c}{\textbf{Source}} & \multicolumn{1}{c}{\textbf{Target}} & & & \\
\textbf{Category} & \multicolumn{1}{c}{\textbf{Singapore}} & \multicolumn{1}{c}{\textbf{Boston}} & \textbf{$p(y)$} & \textbf{$q(y)$} & \textbf{Ratio $w(y)$} \\
\midrule
human.pedestrian.adult & 8,493 & 3,148 & 0.0889 & 0.0599 & 0.6735 \\
human.pedestrian.construction\_worker & 1,469 & 107 & 0.0154 & 0.0020 & 0.1325 \\
human.pedestrian.personal\_mobility & 389 & 1 & 0.0041 & $1.9\times10^{-5}$ & 0.0047 \\
movable\_object.barrier & 17,256 & 3,990 & 0.1807 & 0.0759 & 0.4202 \\
movable\_object.debris & 782 & 111 & 0.0082 & 0.0021 & 0.2579 \\
movable\_object.pushable\_pullable & 988 & 93 & 0.0103 & 0.0018 & 0.1710 \\
movable\_object.trafficcone & 4,024 & 2,703 & 0.0421 & 0.0514 & 1.2206 \\
static\_object.bicycle\_rack & 732 & 560 & 0.0077 & 0.0107 & 1.3901 \\
vehicle.bicycle & 5,222 & 942 & 0.0547 & 0.0179 & 0.3278 \\
vehicle.bus.rigid & 1,925 & 909 & 0.0202 & 0.0173 & 0.8580 \\
vehicle.car & 38,293 & 29,511 & 0.4009 & 0.5614 & 1.4004 \\
vehicle.construction & 2,197 & 752 & 0.0230 & 0.0143 & 0.6220 \\
vehicle.motorcycle & 5,943 & 610 & 0.0622 & 0.0116 & 0.1865 \\
vehicle.trailer & 139 & 1,983 & 0.0015 & 0.0377 & 25.9230 \\
vehicle.truck & 7,660 & 7,143 & 0.0802 & 0.1359 & 1.6945 \\
\midrule
\textbf{TOTAL} & \textbf{95,512} & \textbf{52,563} & \textbf{1.0000} & \textbf{1.0000} & -- \\
\bottomrule
\end{tabular}
\end{table}

\begin{table}[H]
\centering
\caption{Statistics of the original Validation Set of nuImages (larger than  $64 \times 64$ pixels). Used alongside Training Set for category-wise performance monitoring.}
\label{tab:val_stats}
%\hspace*{-1.5cm}
\begin{tabular}{lrrrrr}
\toprule
 & \multicolumn{1}{c}{\textbf{Source}} & \multicolumn{1}{c}{\textbf{Target}} & & & \\
\textbf{Category} & \multicolumn{1}{c}{\textbf{Singapore}} & \multicolumn{1}{c}{\textbf{Boston}} & \textbf{$p(y)$} & \textbf{$q(y)$} & \textbf{Ratio $w(y)$} \\
\midrule
human.pedestrian.adult & 1,851 & 647 & 0.0762 & 0.0584 & 0.7665 \\
human.pedestrian.construction\_worker & 421 & 14 & 0.0173 & 0.0013 & 0.0729 \\
human.pedestrian.personal\_mobility & 87 & 1 & 0.0036 & 0.0001 & 0.0252 \\
movable\_object.barrier & 4,903 & 749 & 0.2018 & 0.0676 & 0.3350 \\
movable\_object.debris & 183 & 20 & 0.0075 & 0.0018 & 0.2397 \\
movable\_object.pushable\_pullable & 227 & 12 & 0.0093 & 0.0011 & 0.1159 \\
movable\_object.trafficcone & 1,112 & 640 & 0.0458 & 0.0578 & 1.2621 \\
static\_object.bicycle\_rack & 177 & 119 & 0.0073 & 0.0107 & 1.4744 \\
vehicle.bicycle & 1,223 & 246 & 0.0503 & 0.0222 & 0.4411 \\
vehicle.bus.rigid & 580 & 242 & 0.0239 & 0.0218 & 0.9150 \\
vehicle.car & 9,383 & 6,547 & 0.3861 & 0.5908 & 1.5301 \\
vehicle.construction & 665 & 151 & 0.0274 & 0.0136 & 0.4979 \\
vehicle.motorcycle & 1,329 & 138 & 0.0547 & 0.0125 & 0.2277 \\
vehicle.trailer & 48 & 294 & 0.0020 & 0.0265 & 13.4318 \\
vehicle.truck & 2,111 & 1,261 & 0.0869 & 0.1138 & 1.3099 \\
\midrule
\textbf{TOTAL} & \textbf{24,300} & \textbf{11,081} & \textbf{1.0000} & \textbf{1.0000} & -- \\
\bottomrule
\end{tabular}
\end{table}

\newpage
\subsection{Computation of $\tau_{CP}$}
\label{app:implementationCP}
\subsubsection{Over a marginal confidence region}
The maximum of $\tau_{CP}(y_0;\wtsrg\LP)$ is achieved at one of the vertices because the property of $\tau_{CP}(y_0, \bomega)$ and its monotonicity in each $\omega_k$, resulting a coordinate-wise procedure to compute its maximum. In our study, we found the exact maximum for $\tau_{CP}(y_0;\wtsrg\LP)$ and for $\tau_{CP}(y_0;\wtsrg\GE)$.

\subsubsection{Over a joint confidence region}\label{app:fp}
The supremum over $\wtsrg$ can be computed exactly using fractional programming iteratively. 
 Denote the score as $r_i$ and the label as $\wt y_i$ in the calibration set for $i=1,...,m_1$.
Given the weight $\bomega^{(t)}$ at time $t$, let $\tau^{(t)}=\tau_{\mathrm{CP}}(\widetilde y_0;\bomega^{(t)})$. By definition, $\tau^{(t)}$ is a $(1-\alpha)$-quantile of a CDF $F(x;\bomega)=\frac{\sum_{i:r_i\le x}\bomega_{\widetilde y_i}}{\sum_{i=1}^{m_1}\bomega_{\widetilde y_i}}$.
Then we update $\bomega^{(t)}$ with $\bomega^{(t+1)}=\arg\min_{\bomega\in\wtsrg}F(\tau^{(t)};\bomega)=\arg\min_{\bomega\in\wtsrg}\frac{\sum_{i:r_i\le \tau^{(t)}}\bomega_{\widetilde y_i}}{\sum_{i=1}^{m_1}\bomega_{\widetilde y_i}}$, which is obtained via fractional programming. If  $F(\tau^{(t)};\bomega^{(t+1)})<1-\alpha$, then $\tau^{(t)}<\tau_{\mathrm{CP}}(\widetilde y_0;\bomega^{(t+1)})$, so the supremum is not reached, and we proceed to the next iteration with $t$. Otherwise, $F(\tau^{(t)};\bomega)\ge1-\alpha$ for every $\bomega\in\wtsrg$, so $\tau^{(t)}$ cannot be improved and is the exact supremum.
In the worst case, the iteration has to be repeated $m_1$ times, which may be computationally intractable. Hence, we implemented an alternative Monte Carlo (MC) method.

\paragraph{Validity of the MC approximation.} Let $\wh\tau_{\mathrm{MC}}(\wt y_0;\wtsrg)\equiv\max_{1\le j\le n_{MC}}\tau_{\mathrm{CP}}(\wt y_0;\bomega_j)$ with $\bomega_j\in\wtsrg$. Since $\wtsrg\subset\wtsrg\LP$,
\bse
\wh\tau_{\mathrm{MC}}(\wt y_0;\wtsrg)\le\tau_{\mathrm{CP}}(\wt y_0;\wtsrg)\le\tau_{\mathrm{CP}}(\wt y_0;\wtsrg\LP).
\ese
The finite-sample guarantee of Theorem~\ref{thm:cp-omega0} holds for the two exact thresholds on the right, but not for $\wh\tau_{\mathrm{MC}}$, which may underestimate the exact supremum. The MC-based results for $\wtsrg$ in Section~\ref{sec:experiments} are therefore an empirical approximation of the exact procedure, while $F_{\mathrm{CP}}(\cdot;\wtsrg\LP)$ is a conservative alternative that is computed exactly and retains the guarantee. The following argument shows that the MC approximation error vanishes as $n_{MC}\to\infty$.

When we compute the maximum over 1,000 Monte Carlo samples from $\wtsrg$, we explain the estimation error as below. $\mathbb{P}[\max_{1\le j\le n_{MC}}\tau_{CP}(y_0;\bomega_j)\le\sup_\bomega\tau_{CP}(y_0;\bomega)-\epsilon]\approx\exp[-f(\sup_\bomega\tau_{CP}(y_0;\bomega))\epsilon n_{MC}]$, where $f$ denotes the PDF of $\tau_{CP}(y_0;\bomega)$ over uniformly sampled $\bomega\in\wtsrg$. For any positive $f(\cdot)$ and $\epsilon$, this probability goes to zero when $n_{MC}\to\infty$. This implies the difference between the Monte Carlo estimated $\tau_{CP}(y_0;\wtsrg)$ and the true $\tau_{CP}(y_0;\wtsrg)$ goes to zero in probability as $n_{MC}\to\infty$. %Thus, the computational cost is proportional to $n_{MC}$.

\newpage
\section{Additional experimental results}
\subsection{Accuracy and expected calibration error (ECE) of classifiers}
\begin{table}[H]
  \caption{Summary of accuracy for 10 classifiers trained on each dataset.}
  \label{tab:acc}
  \begin{center}
      %\resizebox{0.97\linewidth}{!}{
\begin{tabular}{c|c|c|c|c|c}
\toprule
Dataset & AGNews & CIFAR-10 & MNIST & N24News & nuImages\\
 \midrule
Accuracy mean(std) & 94.79\%(0.14) & 91.63\%(0.92) & 97.49\%(0.37) & 85.99\%(0.43) & 95.08\%(0.27)\\
\bottomrule
\end{tabular}%}
  \end{center}
\end{table}

\label{app:ece}
\begin{figure}[H]
    \begin{center}
    \centerline{\includegraphics[width=.85\linewidth]{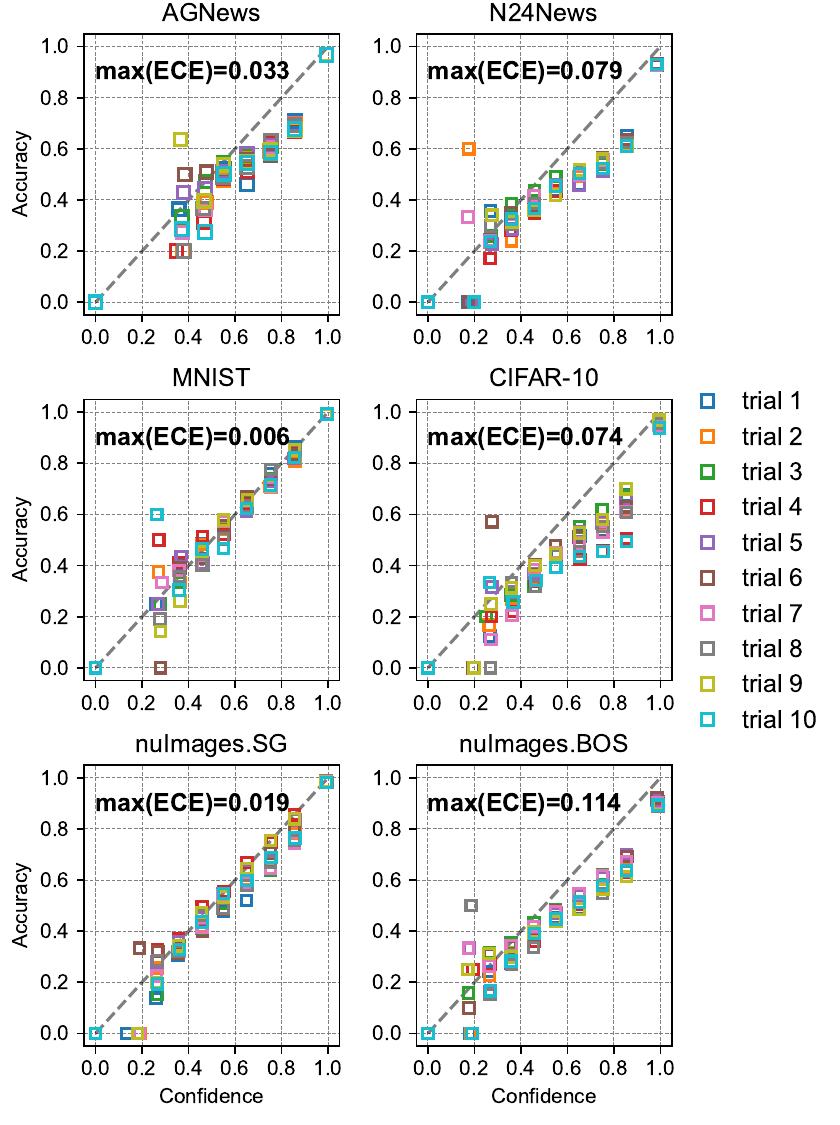}}
        \caption{Expected calibration error (ECE) for each classifier on each dataset. The first column looks all classifiers are well-calibrated, whereas the classifiers on the second column datasets show different levels of uncalibrated performance.}
        \label{fig:ece}
    \end{center}
\end{figure}

\subsection{Tightness of confidence regions}
\label{app:lvres}
\begin{figure}[H]

    \begin{center}
    \centerline{\includegraphics[width=.6\linewidth]{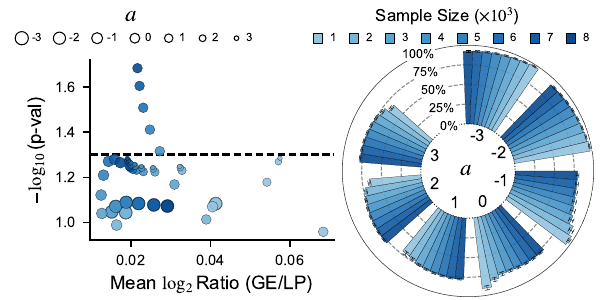}}
        \caption{AGNews. (Left) Volcano plot of the interval length ratios (GE vs. LP) and their corresponding $t$-test significance levels. (Right) Volume ratio of $\wtsrg$ to $\wtsrg\LP$ across varying shift intensities $a$ and sample sizes.}
        \label{fig:agnewsLV}
    \end{center}

\end{figure}

\begin{figure}[H]
    \begin{center}
    \centerline{\includegraphics[width=.6\linewidth]{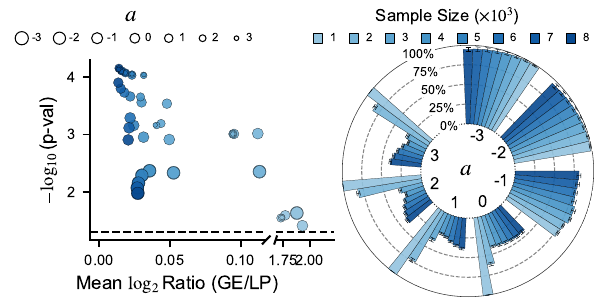}}
        \caption{MNIST. (Left) Volcano plot of the interval length ratios (GE vs. LP) and their corresponding $t$-test significance levels. (Right) Volume ratio of $\wtsrg$ to $\wtsrg\LP$ across varying shift intensities $a$ and sample sizes.}
        \label{fig:ministLV}
    \end{center}
\end{figure}

\begin{figure}[H]
    \begin{center}
    \centerline{\includegraphics[width=.6\linewidth]{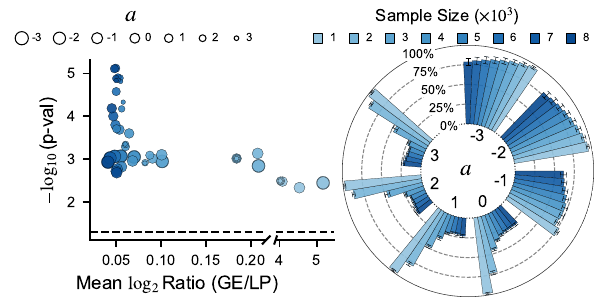}}
        \caption{CIFAR-10. (Left) Volcano plot of the interval length ratios (GE vs. LP) and their corresponding $t$-test significance levels. (Right) Volume ratio of $\wtsrg$ to $\wtsrg\LP$ across varying shift intensities $a$ and sample sizes.}
        \label{fig:cifar10LV}
    \end{center}
\end{figure}

\subsection{Assessment of downstream task}
\label{app:cpres}
\begin{figure}[H]
    \centering
    \includegraphics[width=.75\linewidth]{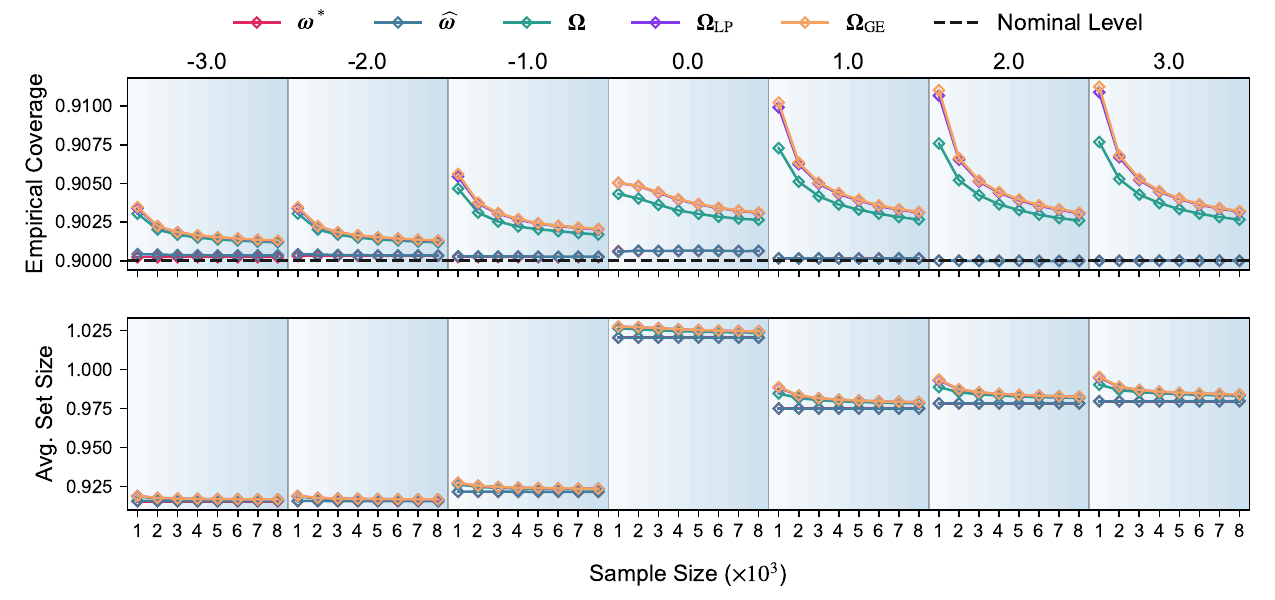}
    \caption{Evaluation of weighted conformal prediction on AGNews. The top panel displays empirical coverage across varying shift intensities ($a$) and sample sizes ($n$). The bottom panel shows average prediction set sizes.}
    \label{fig:cpagnews}
\end{figure}

\begin{figure}[H]
    \centering
    \includegraphics[width=.75\linewidth]{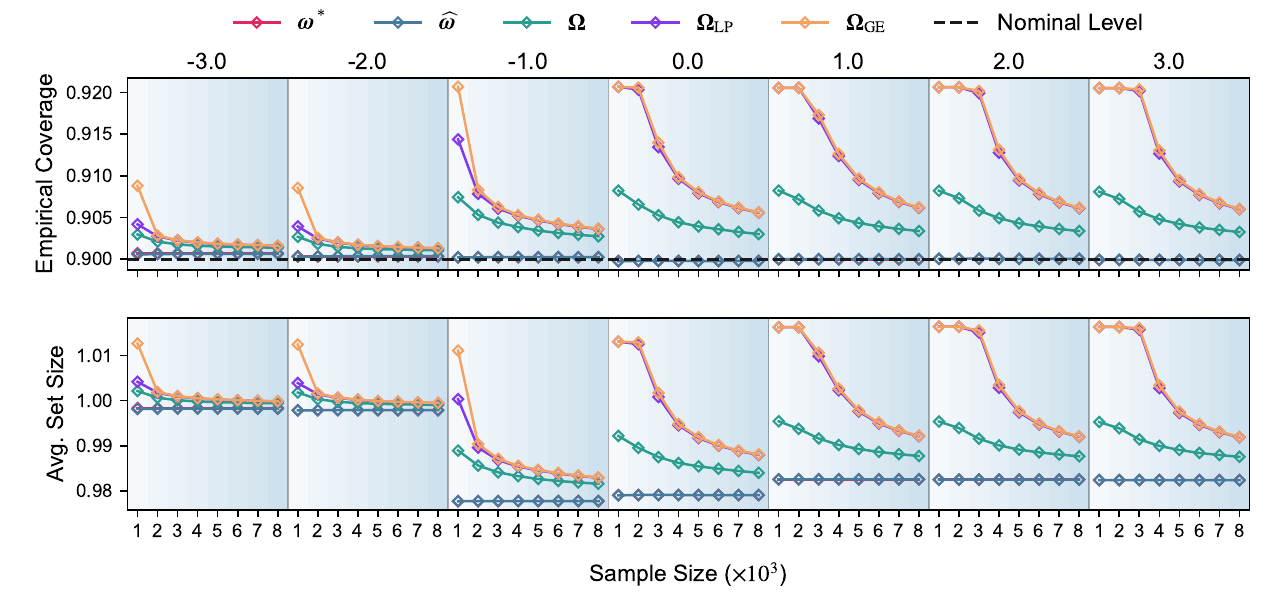}
    \caption{Evaluation of weighted conformal prediction on MNIST. The top panel displays empirical coverage across varying shift intensities ($a$) and sample sizes ($n$). The bottom panel shows average prediction set sizes.}
    \label{fig:cpmnist}
\end{figure}

\begin{figure}[H]
    \centering
    \includegraphics[width=.75\linewidth]{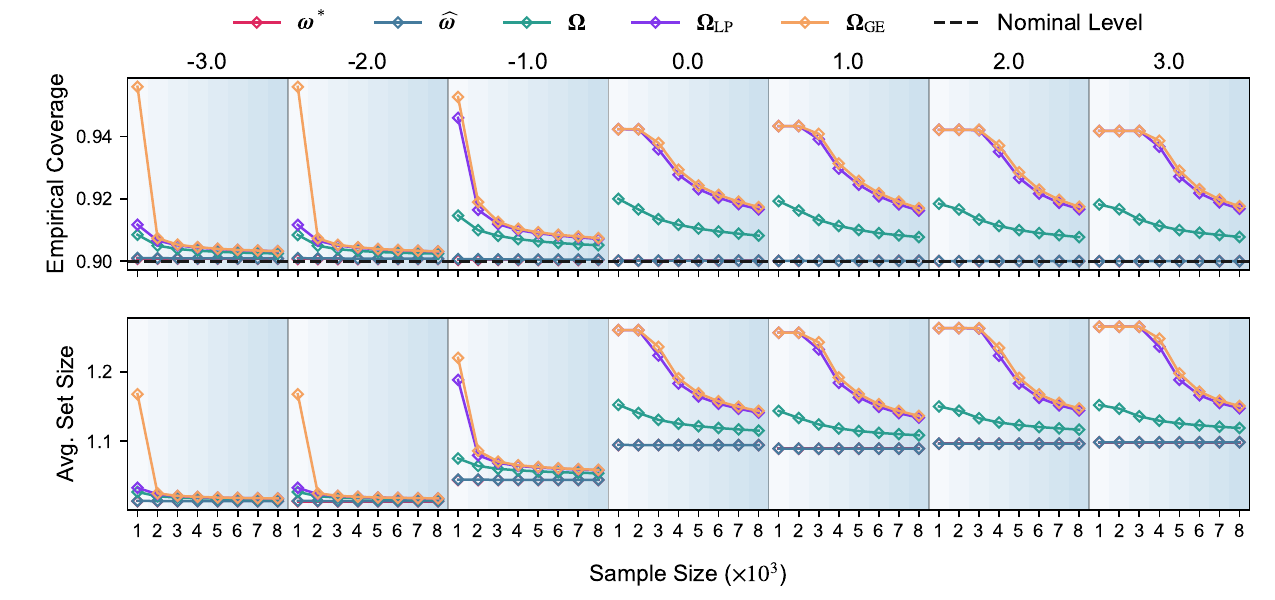}
    \caption{Evaluation of weighted conformal prediction on CIFAR-10. The top panel displays empirical coverage across varying shift intensities ($a$) and sample sizes ($n$). The bottom panel shows average prediction set sizes.}
    \label{fig:cpCifar10}
\end{figure}

\subsection{Runtime and scalability across the number of classes}\label{app:runtime}
\paragraph{Data generating process.} We generate synthetic data from a two-dimensional Gaussian mixture model with $K\in\{2,5,10,20,50,100,200,500\}$ classes. Each class shares the covariance matrix $\boldsymbol{\Sigma}_k=0.5\cdot\mathbf{I}_2$. To keep the classes reasonably separated as $K$ grows, the first mean is placed at the origin and the remaining means on concentric circles of radii $1,2,3,\ldots$, with at most eight equally spaced points (separated by $45^\circ$) per circle. We set $p(y)=1/K$ and draw $q(y)$ from $\mathrm{Dir}(\1_K\cdot10^{-1})$. For each trial, a Naive Bayes classifier $g$ is trained on $2000K$ samples, and $1000K$ samples are drawn for each of the source and target domains with the stratified sampling procedure of Section~\ref{sec:experimental_design}. Each $K$ is repeated 50 times.

\paragraph{Implementation.} Since the constraints are identical across the $2K$ linear programs in \eqref{eq:lp1}--\eqref{eq:lp2}, we solve them with the HiGHS dual simplex method, warm-starting each program from the previous basis. The solver is called from C++ through \texttt{pybind11} with zero-copy access to the constraint arrays, and the programs are distributed across CPU cores with OpenMP after releasing the Python global interpreter lock. The GE baseline is compiled with \texttt{numba 0.60.0}. All runs use the machine described in Appendix~\ref{app:experiment_details} with 5 CPU cores for LP.

\paragraph{Results.} Table~\ref{tab:runtime} reports the wall-clock time per run. $\wtsrg\LP$ is computed within a fraction of a second for $K\le200$ and in 1.72 seconds at $K=500$. GE is faster for large $K$, but Table~\ref{tab:geometry_comparison} shows that its marginal intervals are much longer: the mean $\log_2$ length ratio $K^{-1}\sum_k\log_2(l_{k,\mathrm{GE}}/l_{k,\mathrm{LP}})$ increases from 0.30 at $K=2$ to about 10.9 for $K\ge200$. Thus $\wtsrg\LP$ trades a moderate amount of computation for substantially tighter intervals.

\begin{table}[ht]
\centering
\caption{Runtime comparison (repetition $=50$, unit: seconds per run).}
\label{tab:runtime}
\small\setlength{\tabcolsep}{4pt}
\begin{tabular}{lrrrrrrrr}\toprule
$K$ & 2 & 5 & 10 & 20 & 50 & 100 & 200 & 500 \\\midrule
LP mean & 0.00023 & 0.00028 & 0.00040 & 0.00089 & 0.01160 & 0.04606 & 0.17202 & 1.72289\\
LP std & 0.00005 & 0.00002 & 0.00005 & 0.00009 & 0.00145 & 0.00498 & 0.01689 & 0.24284\\\midrule
GE mean & 0.00046 & 0.00086 & 0.00180 & 0.00331 & 0.01500 & 0.02974 & 0.03896 & 0.26335\\
GE std & 0.00014 & 0.00009 & 0.00027 & 0.00023 & 0.00084 & 0.00192 & 0.00111 & 0.02558\\\bottomrule
\end{tabular}
\end{table}

\begin{table}[H]
\centering
\caption{Mean $\log_2$ length ratio of GE to LP marginal intervals (repetition $=50$).}
\label{tab:geometry_comparison}
\begin{tabular}{lrrrrrrrr}\toprule
$K$ & 2 & 5 & 10 & 20 & 50 & 100 & 200 & 500 \\\midrule
mean & 0.30 & 7.15 & 8.25 & 9.04 & 10.70 & 10.75 & 10.89 & 10.89\\
std & 0.16 & 0.75 & 0.88 & 0.85 & 0.65 & 0.41 & 0.27 & 0.18\\\bottomrule
\end{tabular}
\end{table}

\newpage
\subsection{Numerical stability under extreme shift and rare classes}\label{app:stability}
\paragraph{Empty regions.} $\wtsrg$ is not empty whenever the BBSE estimate $\wh\bomega=\wh\C^{-1}\wh\q$ exists and satisfies $\wh\bomega\ge\0$. It can be empty when $\{\bomega:\ul\C\bomega\le\ol\q,\ \ol\C\bomega\ge\ul\q\}$ does not intersect the nonnegative orthant. If this happens for a class $j$, we follow the strategy of BBSE and project onto the nonnegative orthant, setting $[\ul\omega_j,\ol\omega_j]=[0,0]$.

\paragraph{Unbounded regions and rare source classes.} If $\ul C_{ij}>0$ for some $i$, then $\bomega\ge\0$ and $\sum_{l}\ul C_{il}\omega_l\le\ol q_i$ imply $\omega_j\le\ol q_i/\ul C_{ij}$. Hence $\wtsrg$ can be unbounded only if $\ul C_{ij}=0$ for all $i$ in some column $j$, which typically reflects a very rare source class. In that case we replace the zero endpoints by a tiny positive value (e.g., $10^{-16}$), and we additionally cap each upper bound at 500. Unbounded regions did not occur in our experiments, and we did not observe any adverse effect of this cap. More generally, a rare source class mainly affects the training of the classifier, which can be mitigated by standard imbalanced-learning techniques such as resampling or class-weighted losses; within MaC-LP, as long as the confusion matrix is full rank, its only effect is a looser upper bound for the corresponding weight.

\paragraph{Near-zero target proportions.} When a target class proportion approaches zero, $\omega_k^*=q(k)/p(k)$ approaches zero rather than becoming large. Our benchmark design already covers this regime: Table~\ref{tab:dirichlet_shift} shows the target proportions generated by $\mathrm{Dir}(\1_K\cdot10^a)$ with $K=10$ (sorted within each draw, 100 repetitions). For $a\in\{-3,-2\}$, most target proportions are essentially zero. Across all benchmark settings, which combine these shifts with source and target sample sizes from 1,000 to 8,000, the LP solver had no failures.

\begin{table}[ht]
\centering
\caption{Target label proportions under the Dirichlet shift design ($K=10$, repetition $=100$).}
\label{tab:dirichlet_shift}
\resizebox{\linewidth}{!}{
\begin{tabular}{rcrrrrrrrrrr}\toprule
$a$ & & 0 & 1 & 2 & 3 & 4 & 5 & 6 & 7 & 8 & 9 \\\midrule
$-3$ & mean & 0.000 & 0.000 & 0.000 & 0.000 & 0.000 & 0.000 & 0.000 & 0.000 & 0.011 & 0.989\\
 & std & 0.000 & 0.000 & 0.000 & 0.000 & 0.000 & 0.000 & 0.000 & 0.000 & 0.056 & 0.056\\
$-2$ & mean & 0.000 & 0.000 & 0.000 & 0.000 & 0.000 & 0.000 & 0.000 & 0.003 & 0.060 & 0.937\\
 & std & 0.000 & 0.000 & 0.000 & 0.000 & 0.000 & 0.000 & 0.001 & 0.014 & 0.119 & 0.124\\
$-1$ & mean & 0.000 & 0.000 & 0.000 & 0.000 & 0.003 & 0.012 & 0.034 & 0.085 & 0.238 & 0.628\\
 & std & 0.000 & 0.000 & 0.000 & 0.001 & 0.007 & 0.023 & 0.040 & 0.070 & 0.127 & 0.194\\
$0$ & mean & 0.009 & 0.021 & 0.034 & 0.049 & 0.064 & 0.085 & 0.107 & 0.141 & 0.196 & 0.294\\
 & std & 0.008 & 0.012 & 0.016 & 0.019 & 0.019 & 0.022 & 0.024 & 0.029 & 0.036 & 0.076\\
$1$ & mean & 0.057 & 0.070 & 0.078 & 0.085 & 0.093 & 0.100 & 0.109 & 0.119 & 0.133 & 0.157\\
 & std & 0.011 & 0.009 & 0.008 & 0.008 & 0.007 & 0.007 & 0.007 & 0.008 & 0.013 & 0.020\\
$2$ & mean & 0.085 & 0.090 & 0.093 & 0.096 & 0.098 & 0.101 & 0.103 & 0.107 & 0.111 & 0.116\\
 & std & 0.005 & 0.003 & 0.003 & 0.002 & 0.002 & 0.002 & 0.003 & 0.003 & 0.004 & 0.004\\
$3$ & mean & 0.095 & 0.097 & 0.098 & 0.099 & 0.100 & 0.100 & 0.101 & 0.102 & 0.103 & 0.105\\
 & std & 0.001 & 0.001 & 0.001 & 0.001 & 0.001 & 0.001 & 0.001 & 0.001 & 0.001 & 0.002\\\bottomrule
\end{tabular}}
\end{table}

\newpage
\subsection{Controlled experiment under label shift violation}\label{app:violation}
When $p(\vx\mid y)\neq q(\vx\mid y)$, the identity $\C^*\bomega^*=\q^*$ no longer holds, so neither the point estimate $\wh\bomega$ nor Theorem~\ref{thm:cp-omega0} has a formal guarantee. To understand the empirical behavior of the confidence regions in this case, we generate two-dimensional Gaussian mixture data with $K=5$, means $(0,0)$, $(2,2)$, $(-1,2)$, $(4,0)$, $(-2,-2)$, covariance $\mathbf{I}_2$, $p(y)=(0.1,0.2,0.2,0.15,0.35)^\top$ and $q(y)=(0.37,0.23,0.19,0.2,0.01)^\top$. In the violated setting, the target covariance is inflated to $1.2\,\mathbf{I}_2$. Tables~\ref{tab:asumptn_hold} and \ref{tab:asumptn_violate} and Figure~\ref{fig:violation} report the conformal prediction results (mean and standard error), with ``Oracle'' using $\bomega^*=q(y)/p(y)$.

When label shift holds, all methods attain the nominal 90\% coverage, and $\wtsrg$ and $\wtsrg\LP$ approach the oracle as the sample size grows. When it is violated, even the oracle weights and the point estimate fall below the nominal level at every sample size. The coverage of $\wtsrg$ and $\wtsrg\LP$ stays above 90\% for small and moderate sample sizes, because the confidence regions are wide enough to absorb the misspecification, but it declines toward the point-estimate level once the sampling uncertainty becomes small relative to the systematic error. This behavior is an empirical observation for this design and should not be read as a validity guarantee outside the label shift model.

\begin{table}[ht]
\centering
\caption{Conformal prediction when label shift holds (mean (standard error)).}
\label{tab:asumptn_hold}
\resizebox{\linewidth}{!}{
\begin{tabular}{llrrrrr}\toprule
 & Sample size & 50 & 500 & 5000 & 50000 & 100000 \\\midrule
\multirow{4}{*}{Coverage} & Oracle & 0.901 (0.002) & 0.901 (0.002) & 0.900 (0.002) & 0.901 (0.002) & 0.901 (0.002) \\
 & $\wh\bomega$ (BBSE) & 0.897 (0.003) & 0.901 (0.002) & 0.901 (0.002) & 0.901 (0.002) & 0.900 (0.002) \\
 & $\wtsrg$ & 0.936 (0.002) & 0.935 (0.002) & 0.919 (0.002) & 0.907 (0.002) & 0.905 (0.002) \\
 & $\wtsrg\LP$ & 0.939 (0.002) & 0.940 (0.002) & 0.928 (0.002) & 0.909 (0.002) & 0.907 (0.002) \\\midrule
\multirow{4}{*}{Set size} & Oracle & 1.621 (0.008) & 1.621 (0.008) & 1.619 (0.008) & 1.620 (0.008) & 1.621 (0.008) \\
 & $\wh\bomega$ (BBSE) & 1.611 (0.013) & 1.624 (0.008) & 1.618 (0.008) & 1.619 (0.008) & 1.620 (0.008) \\
 & $\wtsrg$ & 1.846 (0.015) & 1.836 (0.014) & 1.721 (0.010) & 1.652 (0.008) & 1.644 (0.008) \\
 & $\wtsrg\LP$ & 1.872 (0.016) & 1.872 (0.016) & 1.777 (0.012) & 1.663 (0.009) & 1.651 (0.008) \\\bottomrule
\end{tabular}}
\end{table}

\begin{table}[ht]
\centering
\caption{Conformal prediction when label shift is violated (mean (standard error)).}
\label{tab:asumptn_violate}
\resizebox{\linewidth}{!}{
\begin{tabular}{llrrrrr}\toprule
 & Sample size & 50 & 500 & 5000 & 50000 & 100000 \\\midrule
\multirow{4}{*}{Coverage} & Oracle & 0.872 (0.002) & 0.873 (0.002) & 0.872 (0.002) & 0.873 (0.002) & 0.872 (0.002) \\
 & $\wh\bomega$ (BBSE) & 0.866 (0.003) & 0.871 (0.002) & 0.870 (0.002) & 0.870 (0.002) & 0.870 (0.002) \\
 & $\wtsrg$ & 0.913 (0.003) & 0.911 (0.002) & 0.891 (0.002) & 0.878 (0.002) & 0.875 (0.002) \\
 & $\wtsrg\LP$ & 0.917 (0.002) & 0.917 (0.002) & 0.901 (0.002) & 0.880 (0.002) & 0.877 (0.002) \\\midrule
\multirow{4}{*}{Set size} & Oracle & 1.619 (0.008) & 1.619 (0.008) & 1.617 (0.008) & 1.618 (0.008) & 1.619 (0.008) \\
 & $\wh\bomega$ (BBSE) & 1.597 (0.012) & 1.611 (0.008) & 1.605 (0.007) & 1.606 (0.007) & 1.607 (0.007) \\
 & $\wtsrg$ & 1.840 (0.015) & 1.827 (0.014) & 1.705 (0.010) & 1.639 (0.008) & 1.631 (0.008) \\
 & $\wtsrg\LP$ & 1.865 (0.015) & 1.865 (0.015) & 1.759 (0.012) & 1.649 (0.008) & 1.637 (0.008) \\\bottomrule
\end{tabular}}
\end{table}

\begin{figure}[ht]
    \centering
    \includegraphics[width=.46\linewidth]{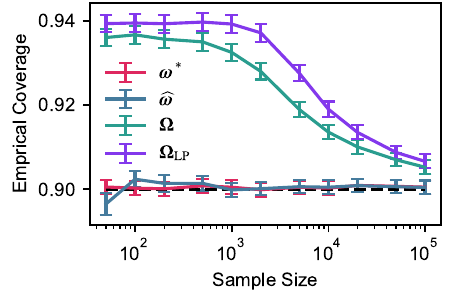}
    \includegraphics[width=.46\linewidth]{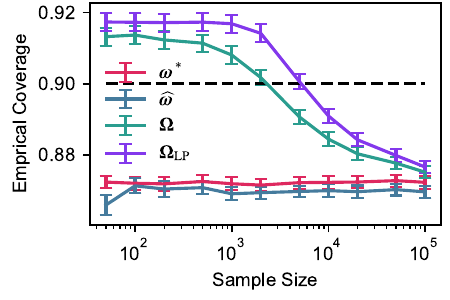}
    \caption{Empirical coverage (mean $\pm$ standard error) against sample size when label shift holds (left) and when the target covariance is inflated to $1.2\,\mathbf{I}_2$ (right).}
    \label{fig:violation}
\end{figure}

%%%%%%%%%%%%%%%%%%%%%%%%%%%%%%%%%%%%%%%%%%%%%%%%%%%%%%%%%%%%
\clearpage
\newpage

\end{document}